\documentclass{article}

\usepackage{amsmath,amssymb,amsthm}

\usepackage{graphicx}
\usepackage{booktabs}
\usepackage{multirow}
\usepackage{float}
\usepackage{threeparttable}

\usepackage{tikz}
\usetikzlibrary{positioning,arrows.meta,backgrounds,fit}

\usepackage{microtype}
\usepackage{url}
\usepackage{hyperref}
\hypersetup{
  pdftitle={When LLM Judge Scores Look Good but Best-of-N Decisions Fail},
  pdfauthor={Eddie Landesberg},
  hidelinks
}
\usepackage[capitalize,noabbrev]{cleveref}
\usepackage{enumitem}

\usepackage[margin=1in]{geometry}

\newtheorem{theorem}{Theorem}
\newtheorem{lemma}[theorem]{Lemma}
\newtheorem{proposition}[theorem]{Proposition}
\newtheorem{corollary}[theorem]{Corollary}
\theoremstyle{definition}
\newtheorem{definition}{Definition}

\theoremstyle{remark}

\DeclareMathOperator*{\argmax}{arg\,max}

\DeclareMathOperator{\Corr}{Corr}
\DeclareMathOperator{\Var}{Var}
\DeclareMathOperator{\Cov}{Cov}
\DeclareMathOperator{\E}{\mathbb{E}}
\DeclareMathOperator{\Prob}{\mathbb{P}}

\newcommand{\X}{\mathcal{X}}

\newcommand{\rwithin}{r_{\text{within}}}

\newcommand{\pnt}{p_{\text{nt}}}
\newcommand{\peff}{p_{\text{eff}}}

\newcommand{\tauwithin}{\bar{\tau}_{\text{within}}}

\graphicspath{{figures/}}

\title{When LLM Judge Scores Look Good but Best-of-N Decisions Fail}

\author{Eddie Landesberg}

\date{}

\begin{document}
\maketitle

\begin{abstract}
Large language models are often used as judges to score candidate responses, then validated with a single global metric such as correlation with reference labels. This can be misleading when the real deployment task is best-of-$n$ selection within a prompt.

In a 5,000-prompt best-of-4 benchmark from Chatbot Arena, a judge with moderate global correlation ($r = 0.47$) captures only 21.0\% of the improvement that perfect selection would achieve over random choice. The gap arises because global agreement is driven largely by prompt-level baseline effects, while selection depends on within-prompt ranking: within-prompt correlation is only $\rwithin = 0.27$, and coarse pointwise scoring creates ties in 67\% of pairwise comparisons.

In a matched-pair best-of-2 audit, explicit pairwise judging recovers much of this lost signal, raising recovery from 21.1\% to 61.2\%. For judge-based selection, the relevant audit should report within-prompt signal, tie rates, and recovery/top-1 accuracy, not global agreement alone.
\end{abstract}

\section{Introduction}
\label{sec:introduction}

Practitioners increasingly use reward models and LLM judges for best-of-$n$ selection, reranking, and model iteration \cite{christiano2017deep,stiennon2020learning,ouyang2022training,nakano2021webgpt}. The common sanity check is a single global metric such as correlation, average error, or pairwise win-rate. If that number looks decent (say, $r \approx 0.5$), teams often assume the judge is safe to optimize against. That assumption can fail.

A judge can have moderate global agreement yet still be weak at the real task: choosing the best candidate \emph{for a specific prompt}. In our study, a judge with $r = 0.47$ achieves only 21.0\% recovery, capturing just 21.0\% of the gain that oracle-optimal selection would provide over random choice (\Cref{tab:main-results}).

If you use judge scores for reranking or best-of-$n$ selection within a prompt, this is the failure mode that matters. If your goal is only model-level benchmarking across many prompts, global metrics can still be appropriate.

The root cause is straightforward: global metrics mix prompt-level baseline agreement with within-prompt ranking. Selection needs the second, but global correlation is often dominated by the first. Coarse score bins make this worse: about 67\% of pairwise comparisons are ties (\Cref{fig:tie-structure}).

The same mismatch can appear with binary pairwise labels too; \Cref{sec:metrics} makes that decomposition explicit.

Prompt difficulty is just one instance of this mechanism; any context-level factor that moves all candidates together can inflate global agreement. We treat the core question as one of \emph{within-prompt decision validity}: can a judge identify the better candidate for a specific prompt?

\paragraph{Scope.} Our main results come from a large cross-policy benchmark (5,000 prompts). We also report a smaller within-policy fresh-draw pilot (24 prompts) to match the common deployment setup. The analysis is limited to one-step, fixed-policy proxy use: given a fixed judge, does proxy-based selection improve within-prompt decisions? Repeated optimization under policy shift is outside scope.

\paragraph{Contributions.} We make four practical contributions:
\begin{enumerate}[leftmargin=*]
    \item \textbf{Decision-centric audit}: We measure recovery and top-1 accuracy, not only global agreement.
    \item \textbf{Within-vs-between decomposition}: We separate baseline agreement from true within-prompt ranking signal.
    \item \textbf{Tie mechanism and pairwise audit}: We show how coarse pointwise scores create ties and test pairwise judging as a head-to-head fix.
    \item \textbf{Cross-judge replication and deployment thresholds}: We replicate the pattern across judge families and translate within-prompt signal into rough recovery targets.
\end{enumerate}

We report routing and efficient-estimation analyses as secondary extensions. Additional PPE-MATH and LLM-to-LLM evaluations appear in the appendix (\Cref{app:extra-generalization}). We return to the surrogate-validity framing only later, in the discussion, once the operational problem is clear.

\section{Setting and Task}
\label{sec:setting}

We study best-of-$n$ selection: given a prompt $x$ and $n$ candidate responses, select the response that maximizes oracle utility. Our strongest empirical claim is about fixed-judge inference-time reranking: sample several candidates for one prompt, score them, and pick one. Verifier-style search is directly analogous. RLHF-style training is only a mechanistic implication here, since prompt-conditional reward differences drive optimization \cite{christiano2017deep,ouyang2022training,rafailov2023direct}; we do not measure training dynamics directly. This deployment focus is increasingly reflected in evaluation benchmarks as well: recent work argues that listwise judge evaluation better matches optimization scenarios than pairwise-only meta-evaluation \cite{wen2026ifrewardbench}.

Our main analysis uses a cross-policy candidate set (responses from different policies). In Appendix~\ref{app:extra-generalization}, we also report a within-policy fresh-draw audit (multiple samples from one generator), which matches common deployment practice. The core requirement is the same in both settings: strong within-prompt ranking signal.

Best-of-$n$ is useful when candidate quality varies meaningfully within a prompt and the selector can detect those differences. It is less useful when candidates are nearly identical, judge signal is weak or tie-heavy, or latency budgets force single-shot responses. This paper focuses on that boundary.

For one prompt, suppose four candidates have oracle quality 0.90, 0.70, 0.69, and 0.68. Oracle best-of-4 picks 0.90. If the judge ties all four at the top, deployed selection is effectively random, with expected quality about 0.74. The global dataset-level correlation can still look good if this same judge tracks easy prompts well. Best-of-$n$ utility is determined by this local within-prompt separation, not by aggregate agreement alone.

Let $x \in \X$ denote a prompt and $i \in \{1, \ldots, n\}$ index candidate responses (from either different policies or multiple draws from one policy). Each response receives a judge score $S_{x,i} \in [0, 100]$ (normalized to $[0, 1]$ for analysis) and a reference label $O_{x,i} \in [0, 1]$.

The decision problem is to select candidate $i^* = \argmax_i S_{x,i}$ (random tie-break among equal maxima) and evaluate by oracle utility $O_{x,i^*}$.

We compare three selection strategies:
\begin{itemize}[leftmargin=*]
    \item Oracle-optimal: $\argmax_i O_{x,i}$ (best achievable)
    \item Random: uniform selection (baseline)
    \item Judge-greedy: $\argmax_i S_{x,i}$ (deployed policy)
\end{itemize}

Best-of-$n$ relies on relative ordering within prompt, not absolute calibration. A judge can look reasonable on average yet still fail at the one decision the deployment actually needs.

\section{Metrics: From ``Looks Correlated'' to ``Decision Useful''}
\label{sec:metrics}

Readers do not need every symbol in this section to follow the paper. The main path uses four quantities: global $r$ (the usual sanity check), within-prompt $r$ (the key signal diagnostic), recovery, and top-1 accuracy. The remaining metrics are supporting diagnostics.

\subsection{Headline Metric (What People Usually Report)}

\paragraph{Global correlation.} $r = \Corr(S, O)$ across all (prompt, candidate) pairs.
This standard sanity check mostly reflects agreement on context-level baseline effects and can exhibit Simpson/ecological-fallacy effects \cite{simpson1951interpretation,robinson1950ecological}. It does \emph{not} directly answer whether the judge picks the right response \emph{within a prompt}, because it conflates between-context and within-prompt association.

\subsection{Decision Validity (What Optimization Needs)}

For deployment, recovery and top-1 accuracy are the two end metrics. The other quantities in this section help diagnose why those numbers are high or low.

\paragraph{Top-1 accuracy (PCS$_n$).} $\Prob(\argmax_i S_{x,i} = \argmax_i O_{x,i})$ with random tie-breaking for tied maxima, corresponding to probability of correct selection (PCS) in ranking-and-selection terms \cite{bechhofer1954single}.

Plain-language: how often the judge picks the oracle-best candidate.

\paragraph{Recovery rate.}
\begin{equation}
\text{Recovery} = \frac{\E[O_{\text{judge}}] - \E[O_{\text{random}}]}{\E[O_{\text{oracle}}] - \E[O_{\text{random}}]}
\end{equation}

Recovery $= 0\%$ means judge is no better than random; $100\%$ means judge matches oracle.

\paragraph{Pairwise sign agreement.} $\Prob(\text{sign}(\Delta S) = \text{sign}(\Delta O) \mid \Delta S \neq 0, \Delta O \neq 0)$

\paragraph{Within-prompt Kendall $\tau$.} A rank-based complement to sign agreement:
\begin{equation}
\tauwithin = \frac{1}{N}\sum_{x=1}^{N} \tau_x^{(b)}
\end{equation}
where $\tau_x^{(b)}$ is Kendall's $\tau$-b for prompt $x$, handling ties appropriately. While sign agreement evaluates pairwise decisions, $\tauwithin$ captures whether the judge recovers the \emph{full ranking} within each prompt.

We treat pairwise sign agreement and Kendall $\tau$ as supporting diagnostics, not as standalone deployment targets. If you only remember one question from this section, it is: does the judge improve the final selection decision over random choice?

\subsection{Supporting Tie Diagnostics}

With discrete judge outputs, ties dominate. We therefore report both conditional agreement ($\pnt$ on non-tied pairs) and tie-aware agreement ($\peff$, counting unresolved ties as 50\%). In practice, $\pnt$ can look reasonable while $\peff$ remains near chance when tie rates are high.

\paragraph{Why tie-aware metrics matter.} If the judge often gives identical top scores, deployment falls back to random tie-breaking. Tie-aware metrics make this explicit.

\subsection{Within-Between Decomposition}

To formalize the gap between global agreement and local choice, we decompose scores into prompt effects and residuals:
\begin{align}
S_{x,i} &= \mu^S_x + \varepsilon^S_{x,i} \\
O_{x,i} &= \mu^O_x + \varepsilon^O_{x,i}
\end{align}
where $\mu^S_x$ and $\mu^O_x$ are context-level means (capturing baseline effects; in this dataset, prompt-level effects), and $\varepsilon^S_{x,i}$ and $\varepsilon^O_{x,i}$ are candidate-specific deviations. Global $r$ reflects correlation of both $\mu$ (baseline effects) and $\varepsilon$ (candidate quality). Optimization depends only on $\varepsilon$.

\paragraph{Binary/pairwise version (same mechanism).}
The same logic holds when labels are binary choices rather than continuous scores. Aggregate agreement still splits into a between-context term and a within-context term. Let $J_{x,a,b}, O_{x,a,b}\in\{0,1\}$ indicate whether candidate $a$ is preferred to $b$ for prompt $x$ by the judge and oracle, respectively. By the law of total covariance:
\begin{equation}
\Cov(J,O)=\Cov(\E[J\mid X],\E[O\mid X]) + \E\!\left[\Cov(J,O\mid X)\right].
\end{equation}
The first term is between-context coupling (baseline effects); the second is within-context directional signal.
Similarly, for a binary oracle:
\begin{equation}
\Var(O)=\Var\!\left(\Pr(O=1\mid X)\right) + \E\!\left[\Pr(O=1\mid X)\left(1-\Pr(O=1\mid X)\right)\right].
\end{equation}
So even with binary labels, aggregate agreement can look strong when context-level baselines dominate, while within-prompt ordering remains weak.

\paragraph{Intuition.}
Suppose many prompts are ``easy'' (one option is obviously better): judge and oracle agree there, inflating global agreement. But best-of-$n$ utility is determined by ``hard'' prompts where candidates are close. If judge signal is weak on those hard prompts, decision quality stays low despite good aggregate numbers.

\Cref{fig:dag} makes this decomposition explicit. Best-of-$n$ only benefits from the candidate-quality path; global correlation reflects both context-level baseline effects and candidate quality:
\begin{enumerate}[nosep]
  \item \textbf{Baseline path} ($D_x \rightarrow S$ and $D_x \rightarrow O$): Context-level baseline effects $D_x$ affect both judge and oracle scores (here, mostly prompt-level effects). This creates global correlation but does not identify which candidate is best within a prompt.
  \item \textbf{Quality path} ($U_{x,i} \rightarrow O$ and $U_{x,i} \rightarrow S$): Within-prompt quality differences $U_{x,i}$ are what optimization needs. The judge edge is attenuated ($\alpha = 0.18$) and further weakened by coarse score quantization.
\end{enumerate}

%
%

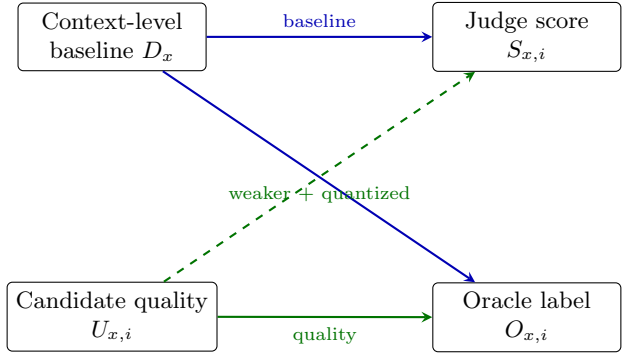
\begin{figure}[t]
  \centering
  \begin{tikzpicture}[
    node distance=2.8cm and 3.0cm,
    every node/.style={font=\small},
    box/.style={rectangle, rounded corners=2pt, draw, minimum width=2.5cm, minimum height=0.9cm, align=center, fill=white},
    baseline/.style={->, >=stealth, thick, blue!70!black},
    quality/.style={->, >=stealth, thick, green!45!black},
    weak/.style={->, >=stealth, thick, dashed, green!45!black}
  ]
    \node[box] (D) {Context-level\\baseline $D_x$};
    \node[box, below=of D] (U) {Candidate quality\\$U_{x,i}$};

    \node[box, right=of D] (S) {Judge score\\$S_{x,i}$};
    \node[box, below=of S] (O) {Oracle label\\$O_{x,i}$};

    \draw[baseline] (D) -- (S) node[midway, above, font=\scriptsize] {baseline};
    \draw[baseline] (D) -- (O);

    \draw[quality] (U) -- (O) node[midway, below, font=\scriptsize] {quality};
    \draw[weak] (U) -- (S) node[midway, below, font=\scriptsize] {weaker + quantized};
  \end{tikzpicture}

  \caption{\textbf{Minimal two-path picture.}
  \emph{Blue solid arrows} are the context-level baseline path ($D_x\!\rightarrow\!S$, $D_x\!\rightarrow\!O$), which is mostly prompt-level in this dataset.
  The \emph{green solid arrow} is the oracle quality link ($U_{x,i}\!\rightarrow\!O$).
  The \emph{green dashed arrow} is the weaker judge quality link ($U_{x,i}\!\rightarrow\!S$), attenuated by noise and score quantization.
  Line style here is conceptual (relative signal strength in this setting), not a statistical significance marker.
  Global correlation uses both paths. Best-of-$n$ depends on the quality path; when $U_{x,i}\!\rightarrow\!S$ is weak, ranking decisions can fail even if global correlation looks acceptable.}
  \label{fig:dag}
\end{figure}

\paragraph{Within-prompt correlation.} $\rwithin = \Corr(\varepsilon^S, \varepsilon^O)$
After removing context-level baseline effects, this measures alignment on relative candidate quality.

\paragraph{Attenuation coefficient.} $\alpha = $ regression slope of $\varepsilon^S$ on $\varepsilon^O$, measuring attenuation of within-prompt oracle signal on the judge-score scale.

\begin{table}[h]
\centering
\caption{Metric definitions, with primary deployment metrics first.}
\label{tab:metric-defs}
\begin{tabular}{lll}
\toprule
\textbf{Metric} & \textbf{Definition} & \textbf{Measures} \\
\midrule
Global $r$ & $\Corr(S, O)$ & Level validity (conflates sources) \\
\text{Top-1 accuracy (PCS$_n$)} & $\Prob(\argmax_i S_{x,i} = \argmax_i O_{x,i})$ & Probability of correct selection \\
$\rwithin$ & $\Corr(\varepsilon^S, \varepsilon^O)$ & Within-prompt coupling \\
$\alpha$ & Slope$(\varepsilon^S \sim \varepsilon^O)$ & Signal attenuation \\
$\pnt$ & $\Prob(\text{agree} \mid \text{both non-tied})$ & Conditional directional accuracy \\
$\peff$ & Tie-adjusted agreement & Decision-aligned accuracy \\
$\tauwithin$ & Mean per-prompt Kendall $\tau$-b & Full ranking quality \\
Recovery & $(O_{\text{judge}} - O_{\text{random}}) / (O_{\text{oracle}} - O_{\text{random}})$ & Decision utility \\
\bottomrule
\end{tabular}
\end{table}

If space is limited in a deployment report, global $r$, within-prompt $r$, recovery, and top-1 accuracy are the minimum set.

\section{Data and Experimental Protocol}
\label{sec:data}

\paragraph{Dataset.} We analyze a 5,000-prompt sample drawn from Chatbot Arena prompts and response rows released through our upstream evaluation pipeline, with protocol context from the MT-Bench/Arena line \cite{zheng2023judging,chiang2024chatbotarena}. For each prompt we have five candidate responses: two from the same Llama-3.3-70B helpful policy (one serves as a duplicate control), one from the same 70B model with a different system prompt, one from a larger Llama-405B helpful policy, and one deliberately bad 70B control. Our main best-of-4 analysis uses the four non-trivial candidates; the bad control is reserved for regime-sensitivity and control analyses. Each released response row carries an already-normalized reference score, \texttt{metadata.oracle\_label} $\in [0,1]$. This is the upstream response-level comparison target bundled with the released rows; in this paper, we take that released scalar as given and use it directly as $O_{x,i}$, with no additional Bradley--Terry refit, vote reconstruction, or post-hoc relabeling.

\paragraph{Model specifications.} The main 5,000-prompt analysis uses a fixed judge snapshot, GPT-5 (gpt-5-2025-08-07), with temperature 0 and 0--100 outputs. We then run a separate multi-judge audit on a shared 500-prompt sample (GPT-5.2, Claude Sonnet 4, GPT-4.1-mini, Gemini-2.5-flash, and Llama-3.3-70B; Section~\ref{sec:multi-judge}), using per-judge complete-case filtering when a score is missing. The main benchmark's reference labels are the stored upstream response-level scores described above. Evaluation prompts are in Appendix~\ref{app:prompts}.

\paragraph{Terminology.} To keep notation compact, we use \emph{oracle} as shorthand for the reference signal in the experiment at hand: these stored response-level labels in the main benchmark, objective correctness on PPE-MATH, and GPT-5.2 labels in the appendix LLM-to-LLM and fresh-draw pilots.

\paragraph{Judge characteristics.} The judge emits only $\sim$20 unique score values (coarse discretization), resulting in a 66.5\% tie rate on similar-candidate comparisons.

\paragraph{Statistical inference.} All confidence intervals are 95\% via cluster bootstrap with prompt as the unit of resampling (1,000 resamples). This accounts for within-prompt dependence across candidates.

\section{Main Results: Correlation Is Not Decision Utility}
\label{sec:results}

\subsection{The Headline Gap}

Unless stated otherwise, results in this section are from the main cross-policy benchmark (5,000 prompts, best-of-4).

\begin{table}[h]
\centering
\caption{Main cross-policy best-of-4 results on 5,000 prompts (four similar-quality candidates per prompt).}
\label{tab:main-results}
\begin{tabular}{lll}
\toprule
\textbf{Metric} & \textbf{Point Estimate} & \textbf{95\% CI} \\
\midrule
Global $r$ & 0.471 & --- \\
Within-prompt $r$ & 0.267 & [0.234, 0.297] \\
$\alpha$ (attenuation) & 0.180 & [0.155, 0.204] \\
$\pnt$ (conditional sign agree) & 67.6\% & [66.3\%, 68.9\%] \\
$\tauwithin$ (ranking) & 0.260 & [0.24, 0.28] \\
$\peff$ (tie-aware) & 55.2\% & --- \\
Recovery rate & 21.0\% & [18.9\%, 23.3\%] \\
Top-1 accuracy & 31.6\% & [30.9\%, 32.3\%] \\
\bottomrule
\end{tabular}
\end{table}

The judge achieves $r = 0.47$ globally but only $r = 0.27$ within prompts. The attenuation coefficient $\alpha = 0.18$ indicates strong attenuation of oracle within-prompt signal on the judge-score scale. The standard audit number looks acceptable, but the deployment-facing numbers do not: recovery is 21.0\% and top-1 accuracy is 31.6\%. For deployment, two points matter most. First, pairwise judging can recover signal hidden by pointwise ties in matched-pair settings (\Cref{tab:pairwise}), but gains are not universal under stricter best-of-4 budgeted audits (\Cref{tab:round-robin-budget}). Second, meaningful best-of-4 gains usually require materially higher within-prompt coupling than $r_{\text{within}} \approx 0.27$ (\Cref{tab:recovery-thresholds}).

\subsection{Variance Decomposition: Shared Baseline Effects Dominate}

\begin{table}[h]
\centering
\caption{Variance decomposition of judge and oracle scores into between-context and within-context components (prompt is the context in this dataset).}
\label{tab:variance}
\begin{tabular}{lrr}
\toprule
\textbf{Component} & \textbf{Judge} & \textbf{Oracle} \\
\midrule
Between-context variance (prompt-level) & 74\% & 81\% \\
Within-context variance (candidate-level) & 26\% & 19\% \\
\bottomrule
\end{tabular}
\end{table}

Most variance is context-level (prompt in this dataset). The judge and oracle agree on baseline effects (for example, some prompts are uniformly lower quality across all candidates). These baseline effects create global correlation but do not help optimization. Global correlation mostly captures agreement on context-level averages: both judge and oracle often agree on baseline shifts across candidate sets. Best-of-$n$ needs a different capability: local choice within one prompt.

That local signal is weak ($\rwithin = 0.27$). Coarse quantization makes it weaker: with only $\sim$20 score levels, many close candidates receive identical scores, so direction is erased and tie-breaking becomes random.

\begin{figure}[h]
\centering
\includegraphics[width=0.8\textwidth]{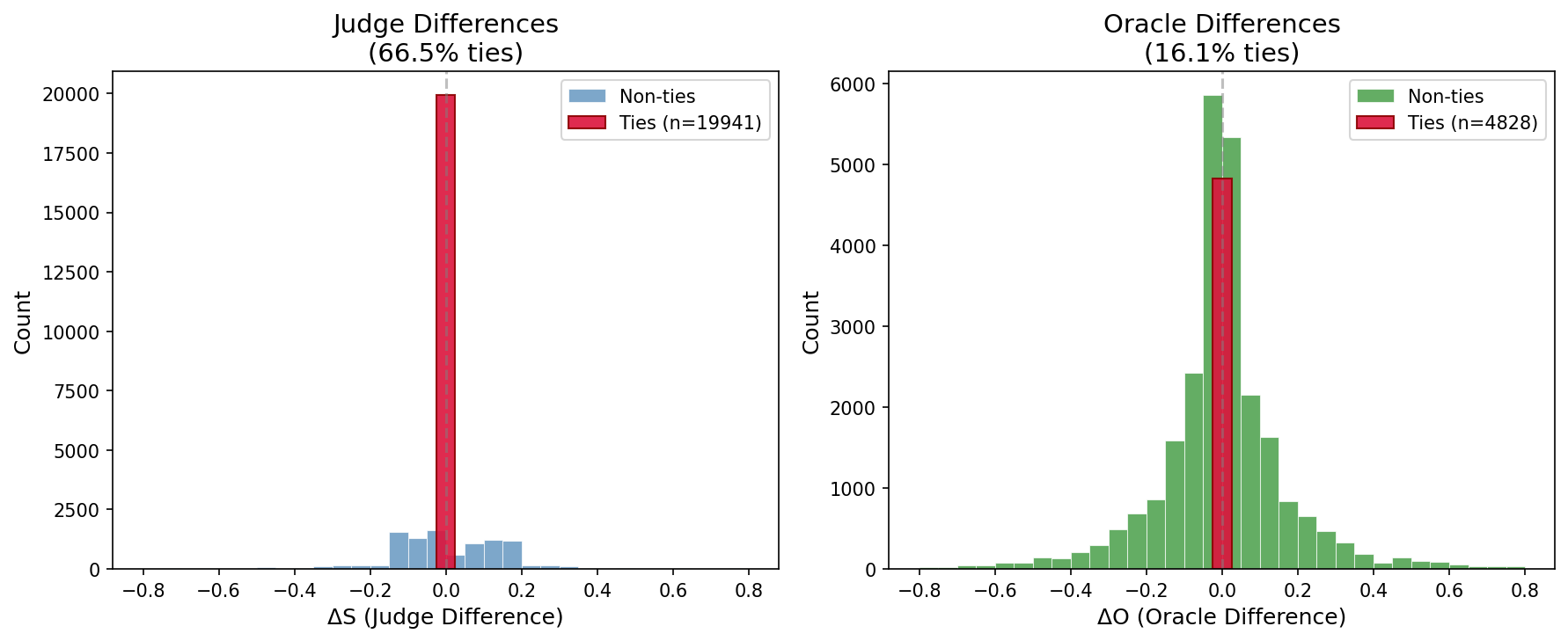}
\caption{Distribution of pairwise score differences. Left: Judge differences ($\Delta S$) show 66.5\% ties due to 20-bin discretization. Right: Oracle differences ($\Delta O$) show 16.1\% ties. The judge's coarse resolution is the primary bottleneck for directional decisions.}
\label{fig:tie-structure}
\end{figure}

\subsection{Tie Structure Is a Major Bottleneck}

The judge emits only $\sim$20 unique values. This coarse discretization means most pairwise comparisons within a prompt produce identical scores. We distinguish pairwise ties (67\%: any two responses receive identical scores) from top-1 margin ties (99\%: the highest-scoring response ties with at least one other). The latter is more severe because best-of-$n$ selection requires a unique winner. With 20 score bins and 4 candidates, ties at the top are nearly universal. On tied pairs, the judge cannot express direction; selection among tied options is effectively random.

\subsection{Does Pairwise Judging Help?}

A natural hypothesis is that the 66.5\% tie rate reflects quantization, not fundamental inability to discriminate. Perhaps forcing the judge to make explicit pairwise comparisons (``Which is better, A or B?'') would recover the lost signal, consistent with recent pairwise ranking evidence for LLM evaluators \cite{qin2024pairwise}. At the same time, pairwise comparison is not universally safer: recent meta-evaluation work shows that comparative formats can also amplify evaluator biases depending on the task and prompt design \cite{jeong2025comparative}.

We test this on 7,000 stratified pairs using the same fixed judge snapshot (GPT-5, gpt-5-2025-08-07), comparing pointwise-derived preferences (from independent 0--100 scores) against an explicit pairwise prompt that allowed \texttt{A}, \texttt{B}, or \texttt{TIE} and also elicited a 1--5 confidence score.

The recovery rates in this subsection (Table~\ref{tab:pairwise}) differ from the main results (Table~\ref{tab:main-results}) for two reasons: (1) pairwise selection is best-of-2, not best-of-4, making the task easier; (2) the 7,000 pairs include 500 control pairs (similar vs.\ unhelpful) that achieve near-perfect accuracy and inflate overall recovery. The relevant comparison is within Table~\ref{tab:pairwise}: pointwise vs.\ pairwise on identical pairs.

\begin{table}[h]
\centering
\caption{Pointwise vs pairwise evaluation on 7,000 matched pairs (best-of-2 setting).}
\label{tab:pairwise}
\begin{tabular}{lrrrr}
\toprule
\textbf{Method} & \textbf{Tie Rate} & \textbf{Agreement} & \textbf{$\peff$} & \textbf{Recovery} \\
\midrule
Pointwise & 59.8\% & 75.2\% & 60.5\% & 21.1\% \\
Pairwise & 3.9\% & 81.1\% & 80.6\% & 61.2\% \\
\bottomrule
\end{tabular}
\end{table}

Pairwise judging reduces ties (3.9\% vs 59.8\%) and improves head-to-head decision utility in this matched-pair dataset: $\peff$ rises from 60.5\% to 80.6\% (+20.1pp), and recovery rises from 21.1\% to 61.2\%. Unlike pointwise scoring, which can abstain via ties, pairwise comparison almost always produces a direction and recovers signal that pointwise quantization had hidden. On the 3,662 pairs where pointwise tied but the oracle had a preference, pairwise returned a non-tie decision 97.9\% of the time and achieved 78.5\% accuracy. Accuracy increased monotonically with oracle gap size (from 66.3\% on small gaps to 94.1\% on large gaps), so the gains from reducing abstention outweighed occasional pairwise errors.

We compared three methods for obtaining judge confidence: self-reported (1--5 scale), token logprobs, and elicited probability estimates (``What is P(A$>$B)?''). In this dataset, elicited probabilities have lower average bin-level calibration error (0.12 vs 0.27 for logprobs), suggesting deliberative probability estimation outperforms raw decoding confidence, in line with prior calibration findings \cite{kadavath2022language,tian2023justask,gao2024bayesian}. See Appendix~\ref{app:confidence} for details.

In this matched best-of-2 setting, pairwise comparison with a capable judge can substantially improve head-to-head decision validity. In this setup, much of the pointwise tie bottleneck appears to come from output quantization rather than inability to discriminate. This finding is for pairwise choices (best-of-2); deployment still requires a task-matched best-of-$n$ audit.

\subsection{Label-Strategy Stress Test: Not Just a Continuous-Scoring Artifact}
\label{sec:label-strategy-ablation}

To test whether our main failure mode is merely an artifact of continuous pointwise scoring, we run a matched best-of-4 ablation with three scoring strategies on the same prompt subset where pairwise labels are available (\(n=3{,}741\) prompts): raw pointwise scores, quantized pointwise scores, and pairwise-Borda aggregation.

\begin{table}[h]
\centering
\caption{Matched best-of-4 label-strategy ablation on prompts with pairwise data (\(n=3{,}741\)).}
\label{tab:label-strategy}
\begin{tabular}{lrrrrr}
\toprule
\textbf{Strategy} & \textbf{Calls/Prompt} & \(\rwithin\) & \textbf{Tie Rate} & \textbf{PCS@4} & \textbf{Recovery} \\
\midrule
Pointwise & 4 & 0.267 & 65.1\% & 31.8\% & 20.7\% \\
Quantized pointwise & 4 & 0.267 & 65.2\% & 31.8\% & 20.7\% \\
Pairwise-Borda (partial) & 6 & 0.359 & 41.2\% & 40.8\% & 51.3\% \\
\bottomrule
\end{tabular}
\end{table}

Two points are clear (\Cref{tab:label-strategy}). First, coarse quantization is not the full story: forcing a 20-bin scale leaves \(\rwithin\), tie rate, and recovery essentially unchanged from raw pointwise scoring. Second, pairwise labeling helps substantially in best-of-4, reducing ties and improving recovery, but does not eliminate the gap to oracle selection.

This pairwise-Borda row is exploratory: available pairwise logs are sparse at the prompt-round-robin level (pairwise data for 74.8\% of prompts, but full 4-candidate round-robin for only 0.02\%). We therefore treat this as directional evidence that label strategy matters, not a definitive replacement for a fully balanced best-of-4 pairwise benchmark. Together with PPE-MATH (\Cref{tab:ppe-comparison}) and external ablations (\Cref{tab:external-label-strategy}), these results support a consistent conclusion: the level-vs-decision gap is not only a continuous-scoring artifact. Label strategy changes magnitude, but the between-vs-within mismatch can persist under objective binary labels and explicit pairwise formulations.

\subsection{Actionable Thresholds: How Much Within-Prompt Signal Is Enough?}
\label{sec:thresholds}

To make deployment decisions concrete, we translate target recovery into required directional quality. Table~\ref{tab:recovery-thresholds} summarizes the threshold mapping from our recovery-requirements analysis (full curve in Appendix~\ref{app:recovery}).

\begin{table}[h]
\centering
\caption{Approximate directional quality needed to reach target best-of-4 recovery in this regime.}
\label{tab:recovery-thresholds}
\begin{tabular}{lrr}
\toprule
\textbf{Target Recovery} & \textbf{Required $\peff$} & \textbf{Required $\rwithin$} \\
\midrule
25\% & 54\% & 0.17 \\
50\% & 60\% & 0.42 \\
75\% & 69\% & 0.76 \\
\bottomrule
\end{tabular}
\end{table}

For best-of-4 in this regime, $\rwithin \approx 0.25$ is weak, $\rwithin \approx 0.4$ is roughly the point where recovery becomes practically meaningful, and very high recovery requires much stronger within-prompt coupling.

\section{Generalizability: Easy Benchmarks Inflate Perceived Quality}
\label{sec:generalizability}

\begin{table}[h]
\centering
\caption{Metrics by evaluation regime (hard similar-only vs mixed with easy unhelpful pairs).}
\label{tab:regimes}
\begin{tabular}{lrrr}
\toprule
\textbf{Regime} & \textbf{Global $r$} & \textbf{Sign Agree} & \textbf{Recovery} \\
\midrule
Similar policies only & 0.47 & 67.6\% & 21.0\% \\
All (incl.\ unhelpful) & 0.82 & 88.2\% & 67.2\% \\
\bottomrule
\end{tabular}
\end{table}

\begin{figure}[h]
\centering
\includegraphics[width=0.9\textwidth]{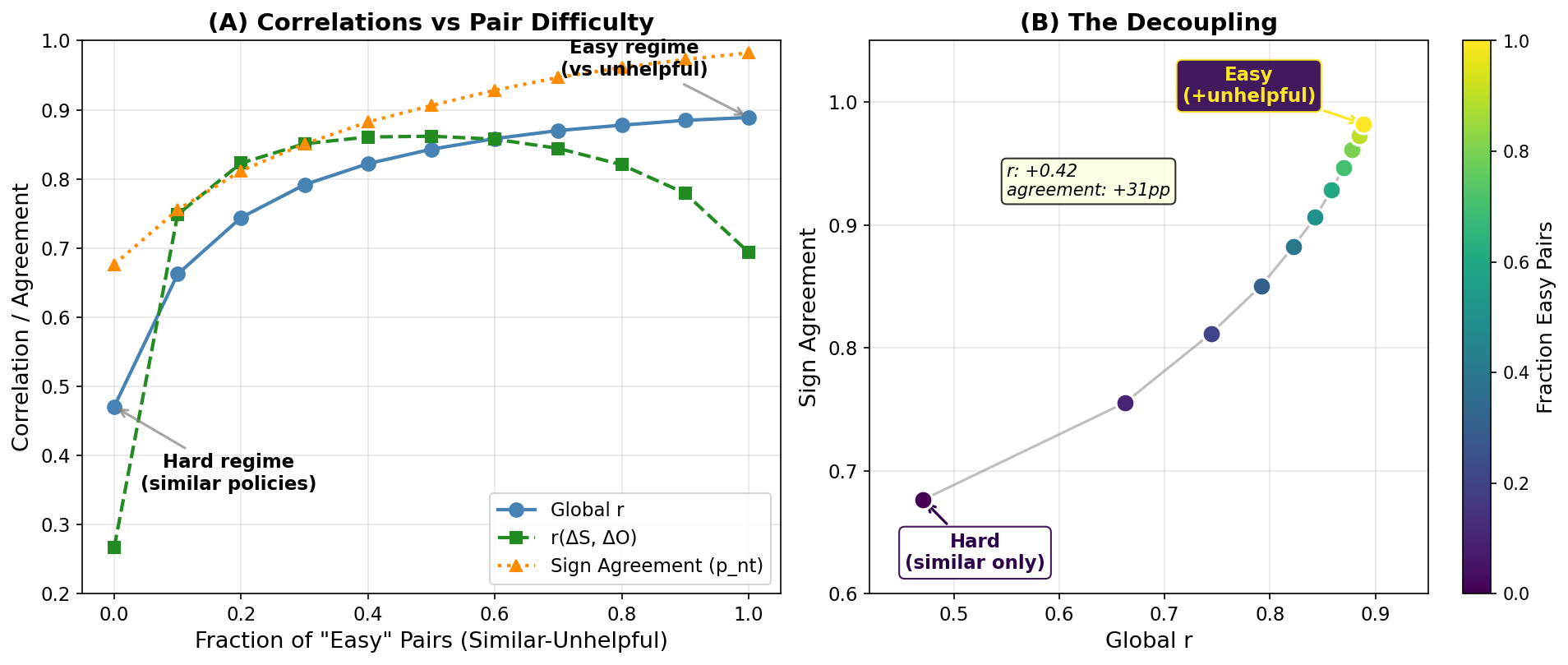}
\caption{Candidate-similarity sensitivity analysis. (A) As trivially distinguishable pairs are added to evaluation, global $r$ can increase from 0.47 to as high as 0.89 while hard-regime performance is unchanged. (B) Global $r$ and sign agreement move differently under this mix shift: adding easy pairs increases both metrics, but sign agreement rises much more (+30pp) than $r$ (+0.42).}
\label{fig:phase-diagram}
\end{figure}

Adding trivially distinguishable pairs (similar vs unhelpful) inflates all metrics. A reported $r = 0.82$ on a benchmark that includes clearly low-quality candidates can substantially overstate performance in the regime that matters, namely distinguishing among similar candidates (\Cref{tab:regimes,fig:phase-diagram}). The similar-unhelpful pairs achieve 98.2\% sign agreement, masking the 67.6\% on similar-similar pairs.

\paragraph{Implication for evaluation design.} Audit judges on the hard regime that matches deployment (near-neighbor candidates), not on mixed-difficulty benchmarks.

\subsection{Generalization Across Judge Families}
\label{sec:multi-judge}

To verify that our findings are not artifacts of a single judge, we evaluate five judges spanning four model families on a shared 500-prompt sample, using complete-case filtering for the small number of missing scores in two judges.

\begin{table}[h]
\centering
\caption{Multi-judge evaluation results on a shared 500-prompt sample (486--500 complete prompts per judge after missing-score filtering). For every judge tested, global $r$ exceeds within-prompt $r$; recovery increases with within-prompt coupling across this small 5-judge sample.}
\label{tab:multi-judge}
\begin{tabular}{llrrrr}
\toprule
\textbf{Judge} & \textbf{Family} & \textbf{Global $r$} & \textbf{Within $r$} & \textbf{Gap} & \textbf{Recovery} \\
\midrule
GPT-5.2 & OpenAI & 0.87 & 0.70 & 20\% & 69.4\% \\
Claude Sonnet 4 & Anthropic & 0.59 & 0.42 & 29\% & 47.7\% \\
GPT-4.1-mini & OpenAI & 0.56 & 0.47 & 17\% & 43.6\% \\
Gemini-2.5-flash & Google & 0.47 & 0.27 & 42\% & 23.8\% \\
Llama-3.3-70b & Meta & 0.31 & 0.23 & 25\% & 18.6\% \\
\bottomrule
\end{tabular}
\end{table}

Across all five judges we tested, global correlation exceeds within-prompt correlation. Recovery rises as within-prompt coupling improves in this five-judge table, but we treat that pattern as descriptive rather than as a stable correlation estimate. The mechanistic conclusion remains anchored in the metric definitions: best-of-$n$ decisions depend on within-prompt ranking quality.

\paragraph{Calibration alone does not reliably repair directional validity.} We applied CJE-style isotonic calibration (monotone, two-stage, and two-stage with covariates) to all judges. Pure monotone calibration improved global $r$ by about 23\% on average but had negligible effect on recovery (about $-0.3$pp), empirically validating Lemma~\ref{lem:monotone}. Adding response length as a covariate improved global $r$ by 38\% and recovery by 2.6pp on average, but degraded sign agreement by 3.7pp. The covariate breaks ties, which improves recovery mechanically, while the tie-breaks are not quality-aligned, which hurts sign agreement. See Appendix~\ref{app:multi-judge} for details.

\paragraph{Additional validation.} We also replicate the pattern on a MATH-only correctness slice from the PPE release introduced by Frick et al.~\cite{ppe2024} (denoted PPE-MATH here), in an LLM-to-LLM setting, and in a deployment-matched within-policy fresh-draw pilot. Those results are reported in Appendix~\ref{app:extra-generalization}; they preserve the same qualitative message while clarifying scope and caveats for each setting.

\section{Discussion}
\label{sec:discussion}

\subsection{Scope of the Evidence}

This paper makes an evaluation-method contribution focused on within-prompt decision validity (directional validity). Aggregate judge agreement can overstate decision usefulness for best-of-$n$, and the paper provides diagnostics that directly target deployment decisions (within-prompt coupling, tie-aware metrics, recovery, and PCS).

The mechanism is broader than prompt difficulty alone: any context-level baseline effect that moves judge and oracle together can inflate global agreement without improving within-prompt choices (e.g., difficulty, topic/domain, response style/length, or candidate-set composition).

Our main evidence is the 5,000-prompt cross-policy benchmark. We also include a deployment-matched within-policy fresh-draw pilot (24 prompts) in Appendix~\ref{app:extra-generalization}; treat that pilot as directional evidence rather than a definitive estimate.

The analysis covers the one-step, fixed-policy case of proxy use: given a fixed judge, does proxy-based selection improve within-prompt decisions? Repeated optimization pressure under policy shift is outside scope.

Some loss is strategy-dependent: in our setup, coarse pointwise scoring creates heavy ties, and explicit pairwise judging can recover substantial head-to-head signal (\Cref{tab:pairwise}). But that gain is not universal: under strict best-of-4 round-robin with explicit token accounting, full pairwise underperforms pointwise for both a smaller judge (GPT-4.1-mini) and a stronger judge pilot (GPT-5) (\Cref{tab:round-robin-budget}). The core level-vs-decision gap is not eliminated by changing labeling style: it persists on PPE-MATH (objective binary correctness labels; \Cref{tab:ppe-comparison}), on external datasets (\Cref{tab:external-label-strategy}), and across judge families (\Cref{tab:multi-judge}). Labeling strategy is therefore a strong moderator, not a complete explanation.

\subsection{Out of Scope}

We measured offline selection with fixed judges. We did not run RLHF training loops or measure learning dynamics over updates. Any RLHF statement in this paper is therefore a motivation-by-analogy claim: if reward models in RLHF pipelines \cite{christiano2017deep,ouyang2022training} have similar variance structure (high context-level baseline signal, weak within-prompt directional signal), optimization would inherit the same signal-quality bottleneck.

The evidence directly covers inference-time best-of-$n$ and reranking decisions with fixed judges. Verifier pipelines that score many candidates and pick one are directly analogous. PPO-style RLHF and DPO-like preference optimization are mechanistically adjacent: learning depends on prompt-conditional reward differences, so weak within-prompt reward signal implies noisier or slower policy improvement even if global RM agreement is high \cite{christiano2017deep,ouyang2022training,rafailov2023direct}. Full training-loop verification remains future work.

\subsection{Use-Case Boundary}

For system-level tasks (ranking whole models, tracking leaderboard shifts), aggregate metrics can be the right target \cite{zheng2023judging,chiang2024chatbotarena,liu2023geval}. Recent meta-evaluation work likewise distinguishes system-ranking quality from instance-level best-output choice \cite{gao2025systemranking}. Those results do not automatically validate within-prompt decisions. If a deployment uses judge scores to choose among candidates for one prompt, directional audits are required. \Cref{tab:use-case-boundary} summarizes this boundary.

\begin{table}[h]
\centering
\small
\begin{tabular}{p{0.47\textwidth} p{0.47\textwidth}}
\toprule
\textbf{System-level benchmarking} & \textbf{Within-prompt optimization} \\
\midrule
Goal: compare models on average across many prompts & Goal: pick the best candidate for one prompt \\
Primary metric: global agreement/correlation is often appropriate & Primary metric must include decision utility (Recovery/PCS) \\
Main risk: judge bias can skew leaderboard estimates & Main risk: global metrics can look strong while local ranking is weak \\
Recommended checks: aggregate agreement + transport/bias diagnostics & Recommended checks: $\rwithin$, tie rates, $\peff$, Recovery, PCS \\
\bottomrule
\end{tabular}
\caption{Use-case boundary for LLM-judge evaluation. Global metrics can be suitable for model-level benchmarking, but within-prompt selection requires directional diagnostics.}
\label{tab:use-case-boundary}
\end{table}

\subsection{A Hierarchy of Evaluation Validity}

Together with causal judge evaluation work, our results suggest four distinct validity levels. These levels are independent: passing one does not guarantee passing the others (Appendix~\ref{app:level-directional}, Theorem~\ref{thm:independence-level-directional}).

The surrogate-validity lens remains a secondary interpretive bridge, while the primary operational target here is within-prompt decision validity.

\begin{table}[h]
\centering
\small
\begin{tabular}{@{}p{0.15\textwidth} p{0.30\textwidth} p{0.45\textwidth}@{}}
\toprule
\textbf{Level} & \textbf{Guarantee} & \textbf{Does not guarantee} \\
\midrule
Correlation validity & High global agreement with oracle labels & Reliable within-prompt ranking or best-of-$n$ gains \\
Level validity & Unbiased policy-level averages after calibration & Correct candidate-level ordering within a prompt \\
Directional validity & Positive within-prompt ranking signal & Large end-to-end decision gains (Recovery/PCS) \\
Decision validity & Real utility in best-of-$n$ selection (high Recovery/PCS) & Causal attribution for why utility is high or low \\
\bottomrule
\end{tabular}
\caption{Validity hierarchy at a glance. These levels are complementary, not nested: a judge can pass one and fail another.}
\label{tab:validity-hierarchy}
\end{table}

Use CJE-style transport diagnostics for policy-level comparisons, and directional metrics (within-prompt $r$, tie rates, recovery/PCS) for instance-level selection.

\subsection{Practical Checklist}

\begin{enumerate}[leftmargin=*]
    \item State the decision target. For system-level benchmarking, global metrics may be sufficient. For within-prompt selection, they are not.
    \item Report within-prompt metrics. Include $\rwithin$, tie rates, recovery, and PCS, not just global $r$.
    \item Audit in the hard regime. Use near-neighbor candidate sets that match deployment.
    \item Audit pairwise before deployment. In matched-pair best-of-2, pairwise sharply reduced ties and improved recovery; in a strict budgeted best-of-4 round-robin setting, it did not (\Cref{tab:pairwise,tab:round-robin-budget}). The recent literature is similarly mixed: pairwise can improve comparison quality, but can also amplify evaluator biases \cite{jeong2025comparative}.
    \item Use explicit uncertainty for routing. Margin alone is weak here; CI width or resampling disagreement is stronger.
    \item Use explicit thresholds. In this regime, $\rwithin \approx 0.4$ is roughly where best-of-4 recovery becomes practically meaningful (\Cref{tab:recovery-thresholds}).
\end{enumerate}

\subsection{Limitations and Falsification}

\begin{itemize}[leftmargin=*]
    \item Our main setting uses one judge scoring interface with coarse discretization ($\sim$20 bins).
    \item The within-policy fresh-draw experiment is a small pilot ($n=24$ prompts).
    \item The stored reference labels are still proxies for ``true'' quality.
\end{itemize}

A counterexample would be a judge with similar global correlation but much higher within-prompt coupling and recovery under the same deployment-matched candidate regime.

\section{Conclusion}
\label{sec:conclusion}

Aggregate validity metrics can substantially overstate an LLM judge's decision usefulness for within-prompt optimization. A judge can look acceptable on a single global metric and still make poor best-of-$n$ choices. In our main cross-policy setting, a judge with $r = 0.47$ reaches only 21.0\% recovery (\Cref{tab:main-results}). In our within-policy fresh-draw pilot (common deployment pattern), global agreement is $r_{\text{global}} = 0.70$, but PCS$_4$ is only 27.4\% and recovery is only 2.2\% (\Cref{tab:within-policy-fresh}).

As secondary extensions beyond the main decision-validity result, we provide concrete tools: within-between decomposition, tie-aware directional metrics, and candidate-similarity sensitivity audits. In our matched-pair best-of-2 audit, pairwise judging sharply reduces quantization ties in head-to-head comparisons, but strict best-of-4 budgeted audits show those gains are not automatic (\Cref{tab:round-robin-budget}). Teams still need deployment-matched best-of-$n$ validation. Margin-only routing (``query the oracle when margin is small'') is weak in this dataset. Explicit uncertainty works better: confidence-interval width or resampling disagreement correlates with gain ($r \approx 0.26$) and captures 26\% of optimal routing value.

This means the observed failure is not only a pointwise-scoring artifact: label strategy can improve performance materially, but the level-vs-decision gap can still persist under alternative labeling regimes.

This paper addresses one-step fixed-policy proxy use (selection/reranking with a fixed judge). Repeated optimization under policy shift is outside scope and left to future work.

Practitioners should report decision-centric metrics, audit in the hard regime that matches deployment, and use explicit uncertainty only when its extra cost is justified. A practical threshold from our results is that best-of-4 recovery becomes meaningfully useful around $\rwithin \approx 0.4$ in this regime (\Cref{tab:recovery-thresholds}). For low-stakes settings, random oracle allocation remains a strong baseline.

This result extends current judge evaluation practice rather than rejecting it: global metrics remain useful for system-level benchmarking, while optimization use cases need directional validation.

\bibliographystyle{plain}
\bibliography{references}

\begin{thebibliography}{10}

\bibitem{bang2005doubly}
Heejung Bang and James~M. Robins.
\newblock Doubly robust estimation in missing data and causal inference models.
\newblock {\em Biometrics}, 61(4):962--973, 2005.

\bibitem{bavaresco2025judgebench}
Anna Bavaresco, Raffaella Bernardi, Leonardo Bertolazzi, Desmond Elliott,
  Raquel Fern{\'a}ndez, Albert Gatt, Esam Ghaleb, Mario Giulianelli, Michael
  Hanna, Alexander Koller, Andre Martins, Philipp Mondorf, Vera Neplenbroek,
  Sandro Pezzelle, Barbara Plank, David Schlangen, Alessandro Suglia, Aditya~K
  Surikuchi, Ece Takmaz, and Alberto Testoni.
\newblock {LLM}s instead of human judges? a large scale empirical study across
  20 {NLP} evaluation tasks.
\newblock In {\em Proceedings of the 63rd Annual Meeting of the Association for
  Computational Linguistics (Volume 2: Short Papers)}, pages 238--255, 2025.

\bibitem{bechhofer1954single}
Robert~E Bechhofer.
\newblock A single-sample multiple decision procedure for ranking means of
  normal populations with known variances.
\newblock {\em The Annals of Mathematical Statistics}, 25(1):16--39, 1954.

\bibitem{chen2024judgebias}
Guiming~Hardy Chen, Shunian Chen, Ziche Liu, Feng Jiang, and Benyou Wang.
\newblock Humans or {LLM}s as the judge? a study on judgement bias.
\newblock In {\em Proceedings of the 2024 Conference on Empirical Methods in
  Natural Language Processing}, pages 8301--8327, 2024.

\bibitem{chernozhukov2018double}
Victor Chernozhukov, Denis Chetverikov, Mert Demirer, Esther Duflo, Christian
  Hansen, Whitney Newey, and James Robins.
\newblock Double/debiased machine learning for treatment and structural
  parameters.
\newblock {\em The Econometrics Journal}, 21(1):C1--C68, 2018.

\bibitem{chiang2024chatbotarena}
Wei-Lin Chiang, Lianmin Zheng, Ying Sheng, Anastasios~Nikolas Angelopoulos,
  Tianle Li, Dacheng Li, Banghua Zhu, Hao Zhang, Michael~I. Jordan, Joseph~E.
  Gonzalez, and Ion Stoica.
\newblock Chatbot arena: An open platform for evaluating llms by human
  preference.
\newblock In {\em Proceedings of the 41st International Conference on Machine
  Learning}, volume 235 of {\em Proceedings of Machine Learning Research},
  pages 8359--8388, 2024.

\bibitem{choi2026irtjudge}
Junhyuk Choi, Sohhyung Park, Chanhee Cho, Hyeonchu Park, and Bugeun Kim.
\newblock Diagnosing the reliability of {LLM}-as-a-judge via item response
  theory.
\newblock {\em arXiv preprint arXiv:2602.00521}, 2026.

\bibitem{christiano2017deep}
Paul~F Christiano, Jan Leike, Tom~B Brown, Miljan Martic, Shane Legg, and Dario
  Amodei.
\newblock Deep reinforcement learning from human preferences.
\newblock In {\em Advances in Neural Information Processing Systems},
  volume~30, pages 4299--4307, 2017.

\bibitem{cobbe2021training}
Karl Cobbe, Vineet Kosaraju, Mohammad Bavarian, Mark Chen, Heewoo Jun, Lukasz
  Kaiser, Matthias Plappert, Jerry Tworek, Jacob Hilton, Reiichiro Nakano,
  Christopher Hesse, and John Schulman.
\newblock Training verifiers to solve math word problems.
\newblock {\em arXiv preprint arXiv:2110.14168}, 2021.

\bibitem{dev2026jrh}
Sunishchal Dev, Andrew Sloan, Joshua Kavner, Nicholas Kong, and Morgan Sandler.
\newblock Judge reliability harness: Stress testing the reliability of {LLM}
  judges.
\newblock {\em arXiv preprint arXiv:2603.05399}, 2026.

\bibitem{ppe2024}
Evan Frick, Tianle Li, Connor Chen, Wei-Lin Chiang, Anastasios~N. Angelopoulos,
  Jiantao Jiao, Banghua Zhu, Joseph~E. Gonzalez, and Ion Stoica.
\newblock How to evaluate reward models for {RLHF}.
\newblock {\em arXiv preprint arXiv:2410.14872}, 2024.

\bibitem{gao2023scaling}
Leo Gao, John Schulman, and Jacob Hilton.
\newblock Scaling laws for reward model overoptimization.
\newblock In {\em Proceedings of the 40th International Conference on Machine
  Learning}, volume 202 of {\em Proceedings of Machine Learning Research},
  pages 10835--10866. PMLR, 2023.

\bibitem{gao2025systemranking}
Mingqi Gao, Yixin Liu, Xinyu Hu, Xiaojun Wan, Jonathan Bragg, and Arman Cohan.
\newblock Re-evaluating automatic {LLM} system ranking for alignment with human
  preference.
\newblock In {\em Findings of the Association for Computational Linguistics:
  NAACL 2025}, pages 4605--4629, 2025.

\bibitem{gao2024bayesian}
Yicheng Gao, Gonghan Xu, Zhe Wang, and Arman Cohan.
\newblock Bayesian calibration of win rate estimation with {LLM} evaluators.
\newblock In {\em Proceedings of the 2024 Conference on Empirical Methods in
  Natural Language Processing}, pages 4757--4769, 2024.

\bibitem{goodhart1984problems}
Charles~AE Goodhart.
\newblock Problems of monetary management: The {UK} experience.
\newblock In {\em Monetary Theory and Practice}, pages 91--121. Palgrave
  Macmillan, 1984.

\bibitem{huang2025empirical}
Hui Huang, Xingyuan Bu, Hongli Zhou, Yingqi Qu, Jing Liu, Muyun Yang, Bing Xu,
  and Tiejun Zhao.
\newblock An empirical study of {LLM}-as-a-judge for {LLM} evaluation:
  Fine-tuned judge model is not a general substitute for {GPT}-4.
\newblock In {\em Findings of the Association for Computational Linguistics:
  ACL 2025}, pages 5880--5895, 2025.

\bibitem{jeong2025comparative}
Hawon Jeong, ChaeHun Park, Jimin Hong, Hojoon Lee, and Jaegul Choo.
\newblock The comparative trap: Pairwise comparisons amplifies biased
  preferences of {LLM} evaluators.
\newblock In {\em Proceedings of the 8th BlackboxNLP Workshop: Analyzing and
  Interpreting Neural Networks for NLP}, pages 79--108, 2025.

\bibitem{kadavath2022language}
Saurav Kadavath, Tom Conerly, Amanda Askell, Tom Henighan, Dawn Drain, Ethan
  Perez, Nicholas Schiefer, Zac Hatfield-Dodds, Nova DasSarma, Eli
  Tran-Johnson, et~al.
\newblock Language models (mostly) know what they know.
\newblock {\em arXiv preprint arXiv:2207.05221}, 2022.

\bibitem{kim2023prometheus}
Seungone Kim, Jamin Shin, Yejin Cho, Joel Jang, Shayne Longpre, Hwaran Lee,
  Sangdoo Yun, Seongjin Shin, Sungdong Kim, James Thorne, and Minjoon Seo.
\newblock Prometheus: Inducing fine-grained evaluation capability in language
  models.
\newblock {\em arXiv preprint arXiv:2310.08491}, 2023.

\bibitem{kim2025rethinking}
Sunghwan Kim, Dongjin Kang, Taeyoon Kwon, Hyungjoo Chae, Dongha Lee, and
  Jinyoung Yeo.
\newblock Rethinking reward model evaluation through the lens of reward
  overoptimization.
\newblock In {\em Proceedings of the 63rd Annual Meeting of the Association for
  Computational Linguistics (Volume 1: Long Papers)}, pages 13252--13280, 2025.

\bibitem{king1997solution}
Gary King.
\newblock {\em A Solution to the Ecological Inference Problem: Reconstructing
  Individual Behavior from Aggregate Data}.
\newblock Princeton University Press, 1997.

\bibitem{koo2024cobbler}
Ryan Koo, Minhwa Lee, Vipul Raheja, Jong~Inn Park, Zae~Myung Kim, and Dongyeop
  Kang.
\newblock Benchmarking cognitive biases in large language models as evaluators.
\newblock In {\em Findings of the Association for Computational Linguistics:
  ACL 2024}, pages 517--545, 2024.

\bibitem{lambert2024rewardbench}
Nathan Lambert, Valentina Pyatkin, Jacob Morrison, LJ~Miranda, Bill~Yuchen Lin,
  Khyathi Chandu, Nouha Dziri, Sachin Kumar, Tom Zick, Yejin Choi, Noah~A
  Smith, and Hannaneh Hajishirzi.
\newblock {RewardBench}: Evaluating reward models for language modeling.
\newblock {\em arXiv preprint arXiv:2403.13787}, 2024.

\bibitem{landesberg2026cje}
Eddie Landesberg and Manjari Narayan.
\newblock Causal judge evaluation: Calibrated surrogate metrics for {LLM}
  systems.
\newblock {\em arXiv preprint arXiv:2512.11150}, 2025.

\bibitem{liu2023geval}
Yang Liu, Dan Iter, Yichong Xu, Shuohang Wang, Ruochen Xu, and Chenguang Zhu.
\newblock G-eval: Nlg evaluation using gpt-4 with better human alignment.
\newblock In {\em Proceedings of the 2023 Conference on Empirical Methods in
  Natural Language Processing}, pages 2511--2522, 2023.

\bibitem{manheim2018categorizing}
David Manheim and Scott Garrabrant.
\newblock Categorizing variants of {Goodhart's Law}.
\newblock {\em arXiv preprint arXiv:1803.04585}, 2018.

\bibitem{nakano2021webgpt}
Reiichiro Nakano, Jacob Hilton, Suchir Balaji, Jeff Wu, Long Ouyang, Christina
  Kim, Christopher Hesse, Shantanu Jain, Vineet Kosaraju, William Saunders,
  et~al.
\newblock {WebGPT}: Browser-assisted question-answering with human feedback.
\newblock {\em arXiv preprint arXiv:2112.09332}, 2021.

\bibitem{neyman1934representative}
Jerzy Neyman.
\newblock On the two different aspects of the representative method: The method
  of stratified sampling and the method of purposive selection.
\newblock {\em Journal of the Royal Statistical Society}, 97(4):558--625, 1934.

\bibitem{ouyang2022training}
Long Ouyang, Jeffrey Wu, Xu~Jiang, Diogo Almeida, Carroll Wainwright, Pamela
  Mishkin, Chong Zhang, Sandhini Agarwal, Katarina Slama, Alex Ray, John
  Schulman, Jacob Hilton, Fraser Kelton, Luke Miller, Maddie Simens, Amanda
  Askell, Peter Welinder, Paul Christiano, Jan Leike, and Ryan Lowe.
\newblock Training language models to follow instructions with human feedback.
\newblock In {\em Advances in Neural Information Processing Systems},
  volume~35, pages 27730--27744, 2022.

\bibitem{qin2024pairwise}
Zhen Qin, Rolf Jagerman, Kai Hui, Honglei Zhuang, Junru Wu, Le~Yan, Jiaming
  Shen, Tianqi Liu, Jialu Liu, Donald Metzler, Xuanhui Wang, and Michael
  Bendersky.
\newblock Large language models are effective text rankers with pairwise
  ranking prompting.
\newblock In {\em Findings of the Association for Computational Linguistics:
  NAACL 2024}, pages 1504--1518, 2024.

\bibitem{rafailov2023direct}
Rafael Rafailov, Archit Sharma, Eric Mitchell, Stefano Ermon, Christopher~D
  Manning, and Chelsea Finn.
\newblock Direct preference optimization: Your language model is secretly a
  reward model.
\newblock In {\em Advances in Neural Information Processing Systems},
  volume~36, 2023.

\bibitem{raudenbush2002hierarchical}
Stephen~W Raudenbush and Anthony~S Bryk.
\newblock {\em Hierarchical Linear Models: Applications and Data Analysis
  Methods}.
\newblock Sage, 2002.

\bibitem{robinson1950ecological}
William~S Robinson.
\newblock Ecological correlations and the behavior of individuals.
\newblock {\em American Sociological Review}, 15(3):351--357, 1950.

\bibitem{simpson1951interpretation}
E.~H. Simpson.
\newblock The interpretation of interaction in contingency tables.
\newblock {\em Journal of the Royal Statistical Society: Series B},
  13(2):238--241, 1951.

\bibitem{stiennon2020learning}
Nisan Stiennon, Long Ouyang, Jeff Wu, Daniel~M Ziegler, Ryan Lowe, Chelsea
  Voss, Alec Radford, Dario Amodei, and Paul Christiano.
\newblock Learning to summarize with human feedback.
\newblock {\em Advances in Neural Information Processing Systems},
  33:3035--3046, 2020.

\bibitem{tian2023justask}
Katherine Tian, Eric Mitchell, Allan Zhou, Archit Sharma, Rafael Rafailov,
  Huaxiu Yao, Chelsea Finn, and Christopher~D Manning.
\newblock Just ask for calibration: Strategies for eliciting calibrated
  confidence scores from language models fine-tuned with human feedback.
\newblock In {\em Proceedings of the 2023 Conference on Empirical Methods in
  Natural Language Processing}, pages 5433--5442, 2023.

\bibitem{tsiatis2006semiparametric}
Anastasios~A Tsiatis.
\newblock {\em Semiparametric Theory and Missing Data}.
\newblock Springer, 2006.

\bibitem{wang2024fair}
Peiyi Wang, Lei Li, Liang Chen, Zefan Cai, Dawei Zhu, Binghuai Lin, Yunbo Cao,
  Lingpeng Kong, Qi~Liu, Tianyu Liu, and Zhifang Sui.
\newblock Large language models are not fair evaluators.
\newblock In {\em Proceedings of the 62nd Annual Meeting of the Association for
  Computational Linguistics (Volume 1: Long Papers)}, pages 9440--9450, 2024.

\bibitem{wen2026ifrewardbench}
Bosi Wen, Yilin Niu, Cunxiang Wang, Xiaoying Ling, Ying Zhang, Pei Ke, Hongning
  Wang, and Minlie Huang.
\newblock {IF}-{RewardBench}: Benchmarking judge models for
  instruction-following evaluation.
\newblock {\em arXiv preprint arXiv:2603.04738}, 2026.

\bibitem{xu2025uncertainjudge}
Zhenghao Xu, Qin Lu, Qingru Zhang, Liang Qiu, Ilgee Hong, Changlong Yu, Wenlin
  Yao, Yao Liu, Haoming Jiang, Lihong Li, Hyokun Yun, and Tuo Zhao.
\newblock Ask a strong {LLM} judge when your reward model is uncertain.
\newblock {\em arXiv preprint arXiv:2510.20369}, 2025.

\bibitem{zadrozny2002transforming}
Bianca Zadrozny and Charles Elkan.
\newblock Transforming classifier scores into accurate multiclass probability
  estimates.
\newblock In {\em Proceedings of the Eighth ACM SIGKDD International Conference
  on Knowledge Discovery and Data Mining}, pages 694--699, 2002.

\bibitem{zeng2024llmbar}
Zhiyuan Zeng, Jiatong Yu, Tianyu Gao, Yu~Meng, Tanya Goyal, and Danqi Chen.
\newblock Evaluating large language models at evaluating instruction following.
\newblock In {\em The Twelfth International Conference on Learning
  Representations}, 2024.

\bibitem{zheng2023judging}
Lianmin Zheng, Wei-Lin Chiang, Ying Sheng, Siyuan Zhuang, Zhanghao Wu, Yonghao
  Zhuang, Zi~Lin, Zhuohan Li, Dacheng Li, Eric~P. Xing, Hao Zhang, Joseph~E.
  Gonzalez, and Ion Stoica.
\newblock Judging {LLM}-as-a-judge with {MT-Bench} and {Chatbot Arena}.
\newblock In {\em Advances in Neural Information Processing Systems},
  volume~36, 2023.

\bibitem{zhu2023judgelm}
Lianghui Zhu, Xinggang Wang, and Xinlong Wang.
\newblock {JudgeLM}: Fine-tuned large language models are scalable judges.
\newblock {\em arXiv preprint arXiv:2310.17631}, 2023.

\end{thebibliography}

\appendix
\section{Why Score Calibration Does Not Fix Direction}
\label{sec:calibration}

\begin{table}[h]
\centering
\caption{Main-benchmark calibration illustration for the fixed GPT-5 judge.}
\label{tab:calibration}
\begin{tabular}{lr}
\toprule
\textbf{Method} & \textbf{Sign Agreement} \\
\midrule
No calibration & 66.3\% \\
Monotonic (isotonic) & 66.5\% \\
With covariates (GBM) & 58.8\% \\
\bottomrule
\end{tabular}
\end{table}

Table~\ref{tab:calibration} is the single-judge illustration on the main cross-policy benchmark. Appendix~\ref{app:multi-judge} repeats the same calibration comparison across all five judges and reports the same qualitative pattern in averages.

Monotone calibration preserves rankings by construction, so it cannot improve directional decisions. This includes isotonic calibration \cite{zadrozny2002transforming} and the monotone-only CJE mode \cite{landesberg2026cje}. Covariate-based calibration can \emph{hurt} because it overfits patterns in the calibration set that do not transport to the decision setting.

\paragraph{Core limitation.} Calibration transforms scores but cannot create information. The judge's limited resolution (20 bins) means information is lost before calibration can act on it. No post-hoc transformation can recover directional signal that was never captured.

\section{Additional Generalization Results}
\label{app:extra-generalization}

\subsection{External Validation: Verifiable Tasks}
\label{sec:ppe}

To test generalization beyond open-ended quality evaluation, we replicate our analysis on a 512-prompt MATH/correctness slice from the PPE release introduced by Frick et al.~\cite{ppe2024}; we denote this local slice PPE-MATH. Each prompt has 32 candidate responses with binary correctness labels. This complements broader reward-model benchmark efforts such as RewardBench \cite{lambert2024rewardbench}. This is a different regime: the reference label is objective (correct/incorrect), and most variance is \emph{within prompt} rather than between prompts.

\begin{table}[h]
\centering
\caption{Comparison of evaluation regimes.}
\label{tab:ppe-comparison}
\begin{tabular}{lrr}
\toprule
\textbf{Metric} & \textbf{Our Study} & \textbf{PPE-MATH} \\
\midrule
Global $r$ & 0.47 & 0.38 \\
Within-prompt $r$ & 0.27 & 0.30 \\
Gap (inflation) & 43\% & 22\% \\
Recovery & 21.0\% & 30\% \\
\midrule
Reference-label variance: between-context & 81\% & 30\% \\
Reference-label variance: within-context & 19\% & 70\% \\
\bottomrule
\end{tabular}
\end{table}

The level-vs-decision gap is smaller on PPE-MATH (22\% vs 43\%) because global correlation cannot be inflated as much when oracle variance is mostly within-context. Yet recovery remains modest (30\%), confirming that low decision utility is not specific to one benchmark style.

\paragraph{Uncertainty signal on PPE-MATH.} CI-width routing is weaker but still positive on PPE-MATH ($r = 0.13$ vs $r = 0.26$ in our main data), suggesting partial transfer of uncertainty signal across regimes.

\subsection{Generalization to LLM-to-LLM Evaluation}
\label{sec:llm-llm}

We also test cheap LLM judges against a stronger LLM reference judge (GPT-5.2).

\begin{table}[h]
\centering
\caption{LLM-to-LLM level-vs-decision gap results.}
\label{tab:llm-llm}
\begin{tabular}{lrrrr}
\toprule
\textbf{Judge} & \textbf{Global $r$} & \textbf{Within $r$} & \textbf{Gap} & \textbf{Recovery} \\
\midrule
GPT-5-mini (7$\times$ cheaper) & 0.84 & 0.54 & 35\% & 52\% \\
GPT-4.1-nano (35$\times$ cheaper) & 0.49 & 0.29 & 42\% & 2\% \\
\bottomrule
\end{tabular}
\end{table}

The inflation gap (35--42\%) persists, indicating that level-vs-decision divergence is not specific to human labels. However, this test uses GPT-5.2 as the reference judge, so it measures agreement with a stronger model rather than direct agreement with humans.

\subsection{Within-Policy Fresh-Draw Pilot (Deployment-Matched)}
\label{sec:within-policy-fresh}

A common deployment pattern for best-of-$n$ is \emph{within-policy} sampling: draw multiple stochastic candidates from one generator and select with a judge. We therefore ran a fresh-draw pilot on Chatbot Arena prompts using one generator model (Llama-3.3-70B), one judge (GPT-4.1-mini), and a stronger reference judge (GPT-5.2), with $n=4$ candidates per prompt.

\begin{table}[h]
\centering
\caption{Within-policy fresh-draw pilot (24 prompts, 4 candidates each).}
\label{tab:within-policy-fresh}
\begin{tabular}{lrrrrr}
\toprule
\textbf{Setting} & \textbf{Global $r$} & \textbf{Within $r$} & \textbf{Judge tie rate} & \textbf{PCS$_4$} & \textbf{Recovery} \\
\midrule
Within-policy, fresh draws & 0.70 & 0.29 & 69.4\% & 27.4\% & 2.2\% \\
\bottomrule
\end{tabular}
\end{table}

This pilot follows the same qualitative pattern as the main study: moderate global agreement, weak within-prompt decision quality, and heavy tie effects. Because $n=24$ prompts, treat magnitudes as preliminary.

\subsection{External Label-Strategy Ablation at Larger Scale}
\label{sec:external-label-strategy}

To test whether the pointwise-vs-quantized finding is specific to one dataset, we ran the same ablation on three external candidate sets with larger prompt counts: PPE-MATH (512 prompts, $k=4$), RewardBench2 (300 prompts, $k=4$), and HelpSteer2 (300 prompts, $k=2$). For each dataset, we compare raw pointwise scores to an explicit 20-bin quantized variant under the same best-of-$k$ protocol.

\begin{table}[h]
\centering
\caption{External label-strategy ablation (pointwise vs explicit quantization).}
\label{tab:external-label-strategy}
\begin{tabular}{lrrrr}
\toprule
\textbf{Dataset} & \(\mathbf{k}\) & \(\mathbf{r_{\text{within}}}\) (pt / qtz) & \textbf{Tie rate} (pt / qtz) & \textbf{Recovery} (pt / qtz) \\
\midrule
PPE-MATH (512) & 4 & 0.303 / 0.310 & 37.0\% / 37.0\% & 31.8\% / 31.8\% \\
RewardBench2 (300) & 4 & 0.630 / 0.633 & 19.4\% / 19.4\% & 86.7\% / 86.7\% \\
HelpSteer2 (300) & 2 & 0.618 / 0.618 & 31.7\% / 31.7\% & 49.5\% / 49.5\% \\
\bottomrule
\end{tabular}
\end{table}

Across all three datasets, explicit quantization leaves within-prompt correlation, tie rates, and recovery essentially unchanged. This supports the same conclusion as the main ablation: the level-vs-decision gap is not just a continuous-scoring artifact.

\subsection{Strict Round-Robin Pairwise Audit Under Token Budgets}
\label{sec:round-robin-budget}

Our main pairwise result (\Cref{tab:pairwise}) uses matched best-of-2 pairs and shows large gains from forcing decisions. To test a stricter deployment setting, we ran a full best-of-4 round-robin audit on the same prompt set with GPT-4.1-mini as judge: pointwise (4 calls/prompt) versus pairwise full round-robin (6 calls/prompt), and a token-matched partial pairwise simulation (4 edges/prompt).

\begin{table}[h]
\centering
\small
\caption{Round-robin best-of-4 audit with explicit token accounting (GPT-4.1-mini: 100 prompts; GPT-5: 40 prompts).}
\label{tab:round-robin-budget}
\resizebox{\textwidth}{!}{
\begin{tabular}{llrrr}
\toprule
\textbf{Judge} & \textbf{Setting} & \textbf{Recovery} & \textbf{Total tokens} & \textbf{Gain per 1k tokens} \\
\midrule
\multirow{3}{*}{GPT-4.1-mini} & Pointwise (4 calls/prompt) & 26.2\% & 199,087 & 0.01055 \\
& Pairwise full RR (6 calls/prompt) & 12.3\% & 478,549 & 0.00206 \\
& Pairwise token-matched partial (4 edges/prompt) & 14.5\% & same as pointwise & --- \\
\midrule
\multirow{3}{*}{GPT-5 (pilot)} & Pointwise (4 calls/prompt) & 79.8\% & 180,389 & 0.01303 \\
& Pairwise full RR (6 calls/prompt) & 68.6\% & 370,950 & 0.00545 \\
& Pairwise token-matched partial (4 edges/prompt) & 73.2\% & same as pointwise & --- \\
\bottomrule
\end{tabular}
}
\end{table}

Across both judges, full round-robin pairwise underperforms pointwise on fixed-call recovery and token efficiency in this strict best-of-4 setting. For GPT-5, pairwise narrows the gap under token-matched partial edges but still does not exceed pointwise in this pilot. Combined with \Cref{tab:pairwise}, this indicates that pairwise gains are \emph{model- and regime-dependent}: forcing comparisons can unlock hidden signal, but only when pairwise judgments are sufficiently accurate for the task and budget.

\section{Can We Route to Proxy Oracle When the Judge Is Uncertain?}
\label{sec:mitigation}

A natural mitigation is \emph{oracle routing}: let the judge decide when confidence is high, and query the expensive oracle when extra information is likely to help. We test whether simple confidence signals can identify those high-value prompts.

\subsection{Margin-Based Routing Fails}

We test this by comparing margin-based allocation (query lowest-margin prompts) against random allocation at fixed oracle budgets.

\begin{table}[h]
\centering
\caption{Proxy oracle routing results: margin-based vs random allocation.}
\label{tab:routing}
\begin{tabular}{lrrrr}
\toprule
\textbf{Budget} & \textbf{Random} & \textbf{Margin} & \textbf{Gain-Optimal} & \textbf{Headroom} \\
\midrule
10\% & 0.785 & 0.785 & 0.814 & 101\% \\
25\% & 0.795 & 0.796 & 0.834 & 97\% \\
50\% & 0.812 & 0.814 & 0.845 & 92\% \\
\bottomrule
\end{tabular}
\end{table}

Margin-based routing captures almost none of the available gain (headroom $\approx$ 92--101\%). In this setting, random allocation is essentially as good as margin-only routing.

\subsection{Why Margin Fails: The Cancellation Identity}

Expected gain from querying the oracle decomposes as:
\begin{equation}
\E[\text{gain} \mid W] = \Pr(\text{judge wrong} \mid W) \times \E[\text{gap} \mid \text{wrong}, W].
\end{equation}

For margin to predict gain, it must correlate with at least one factor. We find:

\begin{enumerate}
    \item \textbf{Higher margin $\to$ judge slightly more often correct} ($r = +0.063$): a weak effect.
    \item \textbf{Higher margin $\to$ lower oracle-best value} ($r = -0.134$): some ``confident'' prompts are just uniformly low-quality prompts where all candidates are mediocre.
\end{enumerate}

When the judge is wrong on a high-margin prompt, the loss is often larger. That larger loss roughly cancels the lower error rate, yielding $\text{Corr}(\text{margin}, \text{gain}) = -0.033$.

\paragraph{The U-shaped pattern.} $\E[\text{gain} \mid \text{margin}]$ is approximately U-shaped: highest at the extremes (ties and very high margins), lowest in the middle. This reflects a mixture of regimes:
\begin{itemize}[leftmargin=*]
    \item \textbf{Easy prompts}: clear winner, judge confident and correct $\to$ high margin, low gain.
    \item \textbf{Hard prompts}: all mediocre, judge latches onto a heuristic $\to$ high margin, but if wrong, large gap.
    \item \textbf{Ambiguous prompts}: no strong signal $\to$ low margin, moderate gap.
\end{itemize}

Margin alone cannot separate ``confident because easy'' from ``confident but wrong on a hard prompt.''

\subsection{Adding a Difficulty Proxy Helps Somewhat}

If high margin means different things on easy vs hard prompts, we need a second feature to separate them. Mean judge score across candidates serves as a difficulty proxy (low mean $\to$ hard prompt).

\begin{table}[h]
\centering
\caption{$\E[\text{gain}]$ by margin and mean judge level.}
\label{tab:2d-gain}
\begin{tabular}{lrrrr}
\toprule
 & \multicolumn{4}{c}{\textbf{Mean Judge Level}} \\
\textbf{Margin} & $<$0.7 & 0.7--0.8 & 0.8--0.9 & $>$0.9 \\
\midrule
0 (tie) & 0.124 & 0.092 & 0.081 & 0.052 \\
0--0.05 & --- & --- & 0.044 & 0.043 \\
0.05--0.10 & 0.127 & 0.055 & 0.044 & 0.036 \\
$>$0.10 & 0.063 & 0.076 & 0.103 & --- \\
\bottomrule
\end{tabular}
\end{table}

The pattern is clearer with two features: low mean level (hard prompts) tends to have high expected gain across margins. This 2D feature explains 2.8\% of gain variance vs 0.1\% for margin alone, a 25$\times$ improvement, though still modest.

\begin{table}[h]
\centering
\caption{Routing with 2D features (margin + level).}
\label{tab:2d-routing}
\begin{tabular}{lrrrrr}
\toprule
\textbf{Budget} & \textbf{Random} & \textbf{Margin} & \textbf{2D Model} & \textbf{Optimal} & \textbf{\% Captured} \\
\midrule
10\% & 0.785 & 0.785 & 0.788 & 0.814 & 11\% \\
25\% & 0.795 & 0.795 & 0.802 & 0.834 & 18\% \\
50\% & 0.812 & 0.814 & 0.822 & 0.845 & 31\% \\
\bottomrule
\end{tabular}
\end{table}

The 2D router captures 11--31\% of the optimal gain, a meaningful improvement over margin alone (which captures $\approx$0\%), but substantial headroom remains.

\subsection{Practical Recommendations}

\begin{enumerate}[leftmargin=*]
    \item \textbf{Don't route on margin alone.} The intuition ``query when uncertain'' fails because margin conflates uncertainty with prompt difficulty.
    \item \textbf{Add a difficulty proxy.} Combining margin with mean score (or similar) partially disentangles the signal.
    \item \textbf{Consider improving the sensor.} Better uncertainty estimates may come from:
    \begin{itemize}
        \item Multiple judge samples (disagreement as uncertainty)
        \item Ensemble of different judges
        \item Continuous scores instead of 20-bin quantization
    \end{itemize}
    \item \textbf{Random allocation is a strong baseline.} With current observables, random allocation is near-optimal for the inference problem and only modestly suboptimal for routing.
\end{enumerate}

\paragraph{Key insight.} Confidence and oracle value are different targets. The judge may be somewhat calibrated about ``am I right?'' but poorly calibrated about ``how much would oracle intervention help?'' Margin mostly tracks the first, not the second.

\subsection{Can Elicited Uncertainty Rescue Routing?}

The analysis above uses \emph{implicit} uncertainty (margin derived from point estimates). What if we explicitly ask the judge for uncertainty? We test two approaches:

\begin{enumerate}
    \item \textbf{Resampling}: Generate $K=5$ independent scores per response; use standard deviation as uncertainty.
    \item \textbf{Self-reported CIs}: Ask the judge to provide 95\% confidence intervals alongside point estimates.
\end{enumerate}

\begin{table}[h]
\centering
\caption{Per-prompt correlation of uncertainty measures with gain (N=500 prompts, 95\% bootstrap CIs).}
\label{tab:uncertainty-corr}
\begin{tabular}{lrl}
\toprule
\textbf{Uncertainty Measure} & \textbf{Corr($U$, gain)} & \textbf{95\% CI} \\
\midrule
Margin (baseline) & $+0.01$ & $[-0.10, +0.11]$ \\
Resample std & $+0.26$ & $[+0.16, +0.36]$ \\
CI width & $+0.26$ & $[+0.15, +0.36]$ \\
\bottomrule
\end{tabular}
\end{table}

Both resampling and CI elicitation yield $r \approx 0.26$, substantially stronger than margin's near-zero correlation. This suggests the judge \emph{can} express useful uncertainty, but we need to elicit it directly or estimate it from repeated samples. That direction is consistent with recent work that routes uncertain reward-model cases to a stronger judge rather than trusting a weak proxy uniformly \cite{xu2025uncertainjudge}. The routing table below uses this same 500-prompt subset, so absolute values are not directly comparable to earlier full-dataset routing tables.

\begin{table}[h]
\centering
\caption{Routing with elicited uncertainty on the 500-prompt uncertainty subset (25\% proxy oracle budget).}
\label{tab:uncertainty-routing}
\begin{tabular}{lrrr}
\toprule
\textbf{Method} & \textbf{Value} & \textbf{Lift vs Random} & \textbf{\% of Optimal} \\
\midrule
Random & 0.798 & --- & 0\% \\
Margin & 0.800 & +0.2\% & 5\% \\
CI Width & 0.808 & +1.2\% & 26\% \\
Optimal & 0.836 & +4.7\% & 100\% \\
\bottomrule
\end{tabular}
\end{table}

CI-width routing captures 26\% of the optimal gain, compared to margin's 5\%. This is a meaningful improvement, though it comes with costs: eliciting CIs requires modified prompts and roughly doubles the output tokens. Resampling achieves similar routing performance without requiring prompt modifications.

\paragraph{Additional benefits of resampling.}
\begin{itemize}[leftmargin=*]
    \item \textbf{Tie elimination}: Original judge has 99\% top-1 tie rate (margin = 0) on our 500-prompt subsample; resampled means reduce this to 29\%.
    \item \textbf{Accuracy boost}: On this 500-prompt pilot, using resampled means instead of single-pass scores improves accuracy from 48.4\% to 51.4\% (+6.2\% relative).
\end{itemize}

\paragraph{Adaptive resampling.} Full resampling ($K=5$) is expensive. We test an adaptive variant: (1) score each candidate once, (2) if margin $\geq 0.10$, stop (clear winner), (3) otherwise resample only the top-2 candidates until margin exceeds threshold or query cap ($K_{\max}=3$) is reached. Simulating on the 500-prompt pilot data:

\begin{center}
\begin{tabular}{lrrrr}
\toprule
\textbf{Method} & \textbf{Accuracy} & \textbf{Queries/prompt} & \textbf{Benefit} & \textbf{Cost} \\
\midrule
$K=1$ (baseline) & 48.4\% & 4.0 & 0\% & 0\% \\
Adaptive & 51.0\% & 7.7 & 87\% & 23\% \\
$K=5$ (full) & 51.4\% & 20.0 & 100\% & 100\% \\
\bottomrule
\end{tabular}
\end{center}

Adaptive resampling achieves \textbf{87\% of full resampling's benefit at 23\% of the cost}, a 3.8$\times$ efficiency gain. This makes resampling practical even when judge calls are expensive.

\paragraph{Why does CI width work when margin fails?}

We decompose $\E[\text{gain} \mid U]$ into its two factors for each uncertainty measure $U$:

\begin{table}[h]
\centering
\small
\caption[Decomposition of gain into error rate and gap]{Decomposition of gain: $\E[\text{gain} \mid U] = \Pr(\text{wrong} \mid U) \times \E[\text{gap} \mid \text{wrong}, U]$. Correlations are computed across quintile bins of each uncertainty measure.\protect\footnotemark}
\label{tab:decomposition}
\begin{tabular}{lcccl}
\toprule
\textbf{Uncertainty $U$} & \textbf{Corr($U$, wrong)} & \textbf{Corr($U$, gap$|$wrong)} & \textbf{Corr($U$, gain)} & \textbf{Mechanism} \\
\midrule
Margin & N/A & N/A & $\approx 0$ & No variance\textsuperscript{$\dagger$} \\
CI width & $+0.33$ & $+0.97$ & $+0.96$ & Reinforcement \\
Resample std & $-0.22$ & $+0.98$ & $+0.97$ & Gap dominates \\
\bottomrule
\end{tabular}
\normalsize
\end{table}
\footnotetext{Bin-level correlations (N=5 quintiles) are higher than per-prompt correlations (Table~\ref{tab:uncertainty-corr}, $r \approx 0.26$) because binning averages out noise. Both confirm the same pattern: CI width and resample std predict gain; margin does not.}

\textsuperscript{$\dagger$}Although only 67\% of random within-prompt pairs tie, the \emph{top-1 margin} (max score $-$ second max) is zero 99\% of the time because coarse score bins make it rare for one candidate to score strictly higher than all three others. This makes margin effectively constant and unusable as a routing signal.

CI width exhibits \emph{reinforcement}: higher CI width predicts both (1) the judge is more likely wrong ($r = +0.08$ per-prompt), and (2) the cost of being wrong is larger ($r = +0.34$ per-prompt among wrong cases). Both factors push expected gain in the same direction, yielding a monotone relationship: prompts in the top CI-width quintile have $3.6\times$ higher expected gain than those in the bottom quintile.

Resample std shows partial cancellation (higher std correlates with \emph{fewer} errors), but the gap-when-wrong effect is so strong ($r = 0.98$) that it overwhelms the error-rate effect, still yielding useful routing signal.

\paragraph{Holdout validation.} To verify these results are not overfit, we split the 500 prompts into 350 train / 150 test. The correlations generalize: CI width achieves $r = 0.25$ on train and $r = 0.31$ on test (even stronger). At 25\% oracle budget on the held-out test set, CI-width routing captures 25\% of optimal gain, confirming the signal is real.

\paragraph{Reasoning versus uncertainty.} A natural concern is that CI elicitation simply causes the judge to deliberate more, and any routing signal comes from this extra reasoning rather than uncertainty per se. We test this by comparing three modes: (1) score only (6 tokens), (2) ``think step by step'' + score (287 tokens), and (3) score + CI (21 tokens). On 300 prompts, reasoning \emph{hurts} accuracy (39\% $\to$ 35\%) despite 45$\times$ more tokens, while CI elicitation maintains accuracy and provides routing signal ($r = +0.10$ for CI width vs $r \approx 0$ for reasoning margin). The signal appears to come from uncertainty rather than deliberation.

\paragraph{Calibration caveat.} The judge's self-reported 95\% CIs achieve only 74\% coverage (CIs are too narrow), so numeric uncertainty calibration is imperfect. Still, CI width remains directionally useful for routing. This mirrors the paper's central theme: absolute calibration can be weak while relative decision signal remains usable.

\paragraph{Recommendation.} When routing value justifies extra cost, CI-width routing is much better than margin routing. Point estimates hide uncertainty that can be recovered with explicit elicitation or resampling.

\section{Efficient Estimation Under Partial Labels}
\label{sec:efficient}

Oracle labels are expensive. Can we estimate recovery rate reliably from a subset of labeled prompts? We develop doubly robust estimators following semiparametric efficiency theory \cite{tsiatis2006semiparametric,chernozhukov2018double} (Appendix~\ref{app:eif}) and test whether margin-based oracle allocation provides practical variance reduction.

\subsection{DR Estimation Under Subsampling}

Given budget $b \in (0, 1]$, we query oracle labels for a fraction $b$ of prompts. The AIPW estimator (Eq.~\ref{eq:aipw}) combines a propensity model $\hat{\pi}(W)$ with an outcome model $\hat{m}_\delta(W)$:
\begin{equation}
\widehat{V}(\delta) = \frac{1}{n} \sum_{t=1}^{n} \left[ \widehat{m}_\delta(W_t) + \frac{R_t}{\widehat{\pi}(W_t)} \Big( O_{\delta(W_t)} - \widehat{m}_\delta(W_t) \Big) \right].
\end{equation}

Table~\ref{tab:subsample} shows that the DR estimator remains close to its full-data estimate under partial labels, with appropriately widening confidence intervals as budget decreases.

\begin{table}[h]
\centering
\caption{DR estimation of recovery rate under partial proxy oracle labels (random allocation).}
\label{tab:subsample}
\begin{tabular}{lrrr}
\toprule
\textbf{Budget} & \textbf{Recovery} & \textbf{95\% CI} & \textbf{ESS} \\
\midrule
100\% (full) & 21.0\% & [18.8\%, 23.2\%] & 5000 \\
50\% & 20.0\% & [17.0\%, 23.0\%] & 2545 \\
25\% & 22.2\% & [18.1\%, 26.2\%] & 1287 \\
10\% & 23.2\% & [17.3\%, 29.1\%] & 492 \\
\bottomrule
\end{tabular}
\end{table}

Across budgets from 10\% to 100\%, the DR estimate remains near the full-data estimate (21.0\%), and every interval still excludes 0\% (the random-baseline recovery level). The qualitative finding is unchanged: the judge has weak but nonzero signal, while lower budgets produce noisier point estimates and wider intervals.

\subsection{Does Margin-Based Allocation Reduce Variance?}

The EIF-optimal oracle allocation is $\pi^*(W) \propto \sqrt{\Var(O_{\delta(W)} \mid W)}$: query prompts where outcome uncertainty is highest (Appendix~\ref{app:eif}). We hypothesized that margin-based routing, which queries where the judge is uncertain, approximates this optimal design.

\paragraph{Empirical test.} We compare three allocation strategies:
\begin{enumerate}
    \item \textbf{Random}: $\pi(W) = b$ (uniform)
    \item \textbf{Margin-based}: $\pi(W)$ higher for low-margin prompts
    \item \textbf{Oracle-optimal}: $\pi^*(W) \propto \sqrt{\text{oracle spread}}$ (oracle-informed upper bound)
\end{enumerate}

Table~\ref{tab:allocation} shows the variance ratio (relative to random) for each design:

\begin{table}[h]
\centering
\caption{Variance ratio of allocation strategies vs.\ random (200 simulations).}
\label{tab:allocation}
\begin{tabular}{lrrr}
\toprule
\textbf{Budget} & \textbf{Random} & \textbf{Margin} & \textbf{Optimal} \\
\midrule
25\% & 1.00 & 1.20 & 0.55 \\
50\% & 1.00 & 2.00 & 0.23 \\
\bottomrule
\end{tabular}
\end{table}

\paragraph{Key findings.}
\begin{enumerate}
    \item \textbf{Oracle-optimal allocation achieves large variance reduction}: 45--77\% reduction vs.\ random, confirming the EIF theory works when properly targeted.
    \item \textbf{Margin-based allocation does not reduce variance}: Variance ratio $>1$ indicates margin is a poor proxy for outcome uncertainty in this dataset.
    \item \textbf{The correlation is wrong-signed}: $\text{Corr}(\text{margin}, \text{oracle spread}) = 0.162$ (positive, not negative). EIF optimality requires low margin $\to$ high variance; our data shows the opposite.
\end{enumerate}

\subsection{A Fundamental Limit: The Judge Is Not Calibrated About Its Errors}

The EIF-optimal design requires predicting where the outcome model fails, specifically $\Var(O_{\delta(W)} - m(W) \mid W)$, the residual variance after conditioning on observables.

\paragraph{The proxy chain.} To target high-residual prompts, we need observable features that predict residual variance. The signal degrades through a chain of proxies:
\begin{center}
\begin{tabular}{lcc}
\toprule
\textbf{Relationship} & \textbf{Correlation} & \textbf{Variance Explained} \\
\midrule
judge\_spread $\to$ oracle\_spread & $r = 0.360$ & 13.0\% \\
oracle\_spread $\to$ residual$^2$ & $r = 0.258$ & 6.7\% \\
judge\_spread $\to$ residual$^2$ & $r = 0.160$ & 2.5\% \\
\bottomrule
\end{tabular}
\end{center}

With only 2.5\% of residual variance predictable, the theoretical maximum variance reduction is $0.025 \times 60\% \approx 1.5\%$, which is completely swamped by estimation noise.

\paragraph{The deeper issue.} Achieving a 20\% variance reduction would require
$\text{Corr}(\text{observable}, \text{residual}^2) \geq 0.58$. Our best observable reaches
only $r = 0.16$ (a 3.6$\times$ gap).

In this dataset, this reveals a key empirical property of the judge: \textbf{it is not well calibrated about its own errors}. A well-calibrated judge would be uncertain (low margin) precisely where it is likely to be wrong. Here, errors are close to random with respect to confidence.

\subsection{Connection to Decision Routing}

The analysis above concerns \emph{inference}: efficiently estimating decision value. A distinct problem is \emph{decision routing}: which prompts to query to maximize realized utility. \Cref{sec:mitigation} shows that margin-based routing fails there as well, for a related but distinct reason: $\E[\text{gain} \mid \text{margin}]$ is approximately flat due to confounding between margin and prompt difficulty.

\subsection{Practical Recommendation}

\textbf{Use uniform random subsampling with DR estimation.} At 25\% budget, random allocation provides unbiased estimates with reasonable precision. With the observables we tested, the gap between achievable and optimal variance reduction remains large.

Margin-based strategies fail for \emph{both} problems:
\begin{itemize}
    \item \textbf{Inference}: Margin explains only 2.5\% of residual variance.
    \item \textbf{Decision routing}: $\E[\text{gain} \mid \text{margin}]$ is not decreasing (\Cref{sec:mitigation}).
\end{itemize}

In this dataset, the judge is not well calibrated about its own errors in either sense. Random allocation is simple, unbiased, and achieves nearly the same efficiency as the observable targeting strategies we evaluated.

\section{Related Work}
\label{sec:related}

LLM-as-a-judge benchmarks such as MT-Bench and Chatbot Arena popularized aggregate agreement metrics for evaluator quality \cite{zheng2023judging,chiang2024chatbotarena}. Judge frameworks such as G-Eval, Prometheus, and JudgeLM largely optimize the same target \cite{liu2023geval,kim2023prometheus,zhu2023judgelm}. More recent meta-evaluation work has made the boundary sharper. Benchmarks such as LLMBar and JUDGE-BENCH show that judge quality depends strongly on example type and task, even when average agreement looks respectable \cite{zeng2024llmbar,bavaresco2025judgebench}. Work on automatic system ranking similarly finds that a judge's usefulness for model-level ranking need not align with instance-level best-output choice \cite{gao2025systemranking}. New 2026 reliability diagnostics push further, using item-response models and targeted stress tests to measure judge consistency under prompt variation and ordinal-grading pressure \cite{choi2026irtjudge,dev2026jrh}. Our contribution targets that deployment mismatch directly: not ``does the judge agree on average,'' but ``does it provide usable within-prompt signal for best-of-$n$ decisions?''

This focus is also consistent with the growing literature on judge bias and brittleness. Prior work documents fairness, position, and cognitive-bias effects in LLM evaluators, and shows that fine-tuned judge models can look strong in narrow settings while failing to generalize as reliable substitutes for stronger frontier judges \cite{wang2024fair,koo2024cobbler,chen2024judgebias,huang2025empirical}. We build on that line, but ask a different question: even if average agreement is acceptable, does the judge induce the right decision for a single prompt?

Our pairwise results sit between two strands of recent evidence. Pairwise prompting and pairwise preference elicitation can improve comparative ranking quality relative to coarse pointwise scoring \cite{qin2024pairwise}. But pairwise comparison is not a free fix: recent work also shows that comparative setups can amplify evaluator biases or behave differently across tasks \cite{jeong2025comparative}. This matches our empirical pattern: pairwise helps substantially in matched best-of-2, but does not automatically win in strict best-of-4 budgeted audits.

The deployment setting we analyze (optimize by selecting high-scored candidates) is central to RLHF and verifier workflows \cite{christiano2017deep,stiennon2020learning,ouyang2022training,cobbe2021training}. Our framing is complementary to Causal Judge Evaluation (CJE) \cite{landesberg2026cje}. CJE targets \emph{level validity} for policy-level comparisons under transport, while we target \emph{directional/decision validity} for instance-level selection. A judge can satisfy one and fail the other; both diagnostics are needed for deployment.

The decomposition also connects to classic multilevel/statistical cautions: aggregate correlation can mislead about within-group relationships (Simpson/ecological effects) \cite{simpson1951interpretation,robinson1950ecological,king1997solution}, and hierarchical modeling emphasizes explicit within-vs-between decomposition \cite{raudenbush2002hierarchical}. In the optimization setting, this appears as strong global agreement driven by shared context effects while within-set ranking remains weak.

Finally, our findings align with recent reward-model evaluation concerns. Reward overoptimization and Goodhart effects show that proxy optimization can fail even when proxies look strong globally \cite{gao2023scaling,goodhart1984problems,manheim2018categorizing}. RewardBench, recent work on reward-model evaluation, and overoptimization-focused analyses all strengthen the same motivation for deployment-matched diagnostics rather than relying on one headline metric \cite{lambert2024rewardbench,ppe2024,kim2025rethinking}. Recent listwise evaluation work makes the same deployment-matching point explicitly for judge benchmarks, arguing that multi-candidate instruction-following evaluation is a better proxy for optimization use than pairwise-only setups \cite{wen2026ifrewardbench}. Our paper isolates an earlier bottleneck in that pipeline: low within-prompt signal and heavy ties can limit best-of-$n$ utility before aggressive optimization begins.

\section{Bootstrap Details}
\label{app:bootstrap}

All confidence intervals use cluster bootstrap with the prompt as the unit of resampling. We resample prompts with replacement ($n = 5,000$), compute the statistic on each bootstrap sample, and report the 2.5th and 97.5th percentiles as the 95\% CI. This accounts for within-prompt dependence across the four similar candidates used in the main best-of-4 benchmark.

\section{Additional Stratifications}
\label{app:stratifications}

The judge shows heterogeneous reliability across conditions:
\begin{itemize}[leftmargin=*]
    \item \textbf{Response length}: Accuracy is 63.7\% on short responses vs 72.8\% on long responses (+9.1\%).
    \item \textbf{Judge margin}: Higher margins weakly predict correctness, but not routing value.
    \item \textbf{Length bias}: Length does not systematically bias judge errors (null result).
\end{itemize}

\section{Recovery Requirements Curve}
\label{app:recovery}

\begin{figure}[h]
\centering
\includegraphics[width=0.8\textwidth]{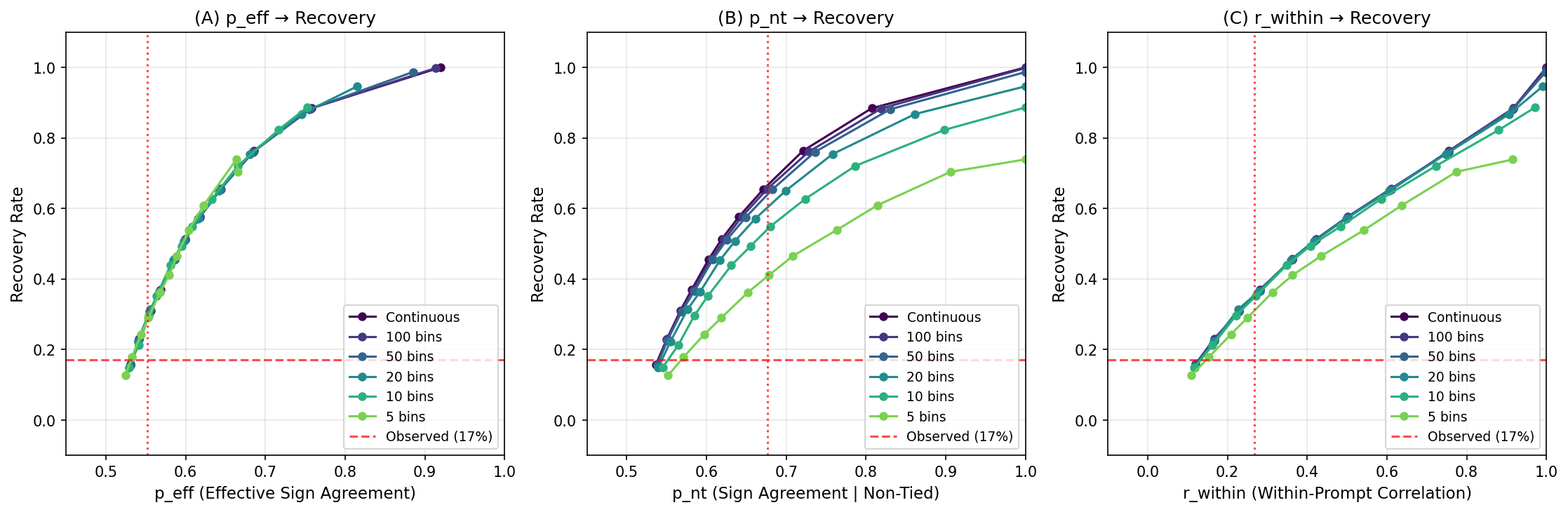}
\caption{What judge quality is needed for $X$\% recovery? The curve shows required $\peff$ and $\rwithin$ as a function of target recovery rate. Current judge ($\rwithin = 0.27$) achieves 21.0\% recovery; reaching 50\% recovery would require $\rwithin \approx 0.42$.}
\label{fig:recovery-requirements}
\end{figure}

\begin{table}[h]
\centering
\caption{Recovery requirements.}
\label{tab:recovery-requirements}
\begin{tabular}{lrr}
\toprule
\textbf{Target Recovery} & \textbf{Required $\peff$} & \textbf{Required $\rwithin$} \\
\midrule
25\% & 54\% & 0.17 \\
50\% & 60\% & 0.42 \\
75\% & 69\% & 0.76 \\
90\% & 76\% & 0.92 \\
\bottomrule
\end{tabular}
\end{table}

\section{Effect of Discretization}
\label{app:discretization}

At fixed underlying noise level, coarser discretization reduces recovery:

\begin{table}[h]
\centering
\caption{Effect of discretization on recovery.}
\label{tab:discretization}
\begin{tabular}{lrr}
\toprule
\textbf{Bins} & \textbf{$\peff$} & \textbf{Recovery} \\
\midrule
Continuous & 61.8\% & 57.6\% \\
100 & 61.8\% & 57.3\% \\
20 & 61.6\% & 54.9\% \\
10 & 60.9\% & 51.1\% \\
5 & 58.9\% & 40.8\% \\
\bottomrule
\end{tabular}
\end{table}

This confirms that the judge's 20-bin discretization is a meaningful bottleneck.

\section{Confidence Elicitation Methods}
\label{app:confidence}

For pairwise head-to-head judgments (and pairwise components inside best-of-$n$ systems), confidence-based routing could improve performance by trusting only high-confidence judgments. We compared three methods for obtaining judge confidence on 490 stratified pairs.

\subsection{Methods Compared}

\begin{enumerate}
    \item \textbf{Self-reported confidence} (1--5 scale): Ask the judge to rate its confidence alongside its choice.
    \item \textbf{Token logprobs}: Extract $P(A)$ and $P(B)$ from the decoding distribution at temperature 1.0, compute confidence as $\max(P_A, P_B) / (P_A + P_B)$.
    \item \textbf{Elicited probability}: Ask the judge directly: ``What is the probability that A is better than B?'' and parse the numeric response.
\end{enumerate}

\subsection{Calibration Results}

We summarize calibration by binning stated/implied probabilities and comparing each bin's mean stated probability to its observed A-win rate. The scalar summary reported below is the unweighted mean of those bin-level absolute errors.

\begin{table}[h]
\centering
\caption{Confidence elicitation calibration (490 pairs).}
\label{tab:confidence-calibration}
\begin{tabular}{lrrr}
\toprule
\textbf{Method} & \textbf{Stated P(A)} & \textbf{Observed A-win} & \textbf{Error} \\
\midrule
\multicolumn{4}{l}{\textit{Elicited P(A$>$B)}} \\
\quad $P \in [0.0, 0.2]$ & 0.17 & 0.12 & 0.05 \\
\quad $P \in [0.2, 0.4]$ & 0.29 & 0.41 & 0.12 \\
\quad $P \in [0.4, 0.6]$ & 0.43 & 0.53 & 0.10 \\
\quad $P \in [0.6, 0.8]$ & 0.69 & 0.46 & 0.23 \\
\quad $P \in [0.8, 1.0]$ & 0.88 & 0.76 & 0.12 \\
\midrule
\multicolumn{4}{l}{\textit{Logprob P(A)}} \\
\quad $P \in [0.0, 0.2]$ & 0.03 & 0.31 & 0.29 \\
\quad $P \in [0.2, 0.4]$ & 0.31 & 0.57 & 0.26 \\
\quad $P \in [0.4, 0.6]$ & 0.50 & 0.38 & 0.11 \\
\quad $P \in [0.6, 0.8]$ & 0.69 & 0.36 & 0.33 \\
\quad $P \in [0.8, 1.0]$ & 0.98 & 0.61 & 0.37 \\
\bottomrule
\end{tabular}
\end{table}

\paragraph{Key finding.} Elicited probabilities achieve lower mean bin-level calibration error than logprobs (0.12 vs 0.27, a 2.3$\times$ improvement). The logprob distribution is highly bimodal (most predictions are $P > 0.8$ or $P < 0.2$), reflecting overconfidence at decoding time. Elicited probabilities show more spread and better track true base rates.

\paragraph{Why logprobs fail.} Token-level probabilities reflect the model's confidence \emph{after} it has committed to an answer format. When asked to output just ``A'' or ``B'', the model is very confident in its choice even when that choice is wrong. In contrast, asking for an explicit probability estimate forces deliberation about uncertainty.

\paragraph{Practical recommendation.} For confidence-based routing in pairwise judging, elicit explicit probability estimates rather than relying on token logprobs or self-reported confidence scales, consistent with prior calibration work \cite{kadavath2022language,tian2023justask}.

\section{Formal Results: When Correlation Fails to Predict Recovery}
\label{app:formal}

We formalize two key insights: (1) global correlation does not identify recovery, and (2) optimal routing depends on a quantity that may be unpredictable from observables.

\subsection{Non-Identifiability of Recovery from Global Correlation}

\begin{proposition}[Non-Identifiability]
\label{prop:nonident}
For any global correlation $r \in (0, 1)$ and any recovery targets $\rho_1, \rho_2 \in [0, 1]$ with $\rho_1 < \rho_2$, there exist data-generating processes $P_1$ and $P_2$ such that:
\begin{enumerate}
    \item $\mathrm{Corr}_{P_1}(S, O) = \mathrm{Corr}_{P_2}(S, O) = r$
    \item $\mathrm{Recovery}(P_1) = \rho_1$ and $\mathrm{Recovery}(P_2) = \rho_2$
\end{enumerate}
\end{proposition}

\begin{proof}[Proof sketch]
Consider the multilevel decomposition:
\begin{align}
S_{x,i} &= \mu^S_x + \varepsilon^S_{x,i}, \\
O_{x,i} &= \mu^O_x + \varepsilon^O_{x,i}.
\end{align}

Global correlation decomposes as:
\[
\mathrm{Corr}(S, O) = \frac{\mathrm{Cov}(\mu^S, \mu^O) + \mathrm{Cov}(\varepsilon^S, \varepsilon^O)}{\sqrt{\mathrm{Var}(S) \cdot \mathrm{Var}(O)}}.
\]

Recovery depends on within-prompt ranking, which is governed by $\varepsilon^S, \varepsilon^O$ and tie structure. By varying the ratio of between-prompt to within-prompt variance while holding the weighted sum constant, we can construct processes with identical global $r$ but arbitrarily different within-prompt coupling $r_{\text{within}} = \mathrm{Corr}(\varepsilon^S, \varepsilon^O)$.

For example, let $P_1$ have high between-prompt correlation but zero within-prompt coupling (recovery $\approx 0$), and let $P_2$ have moderate both (high recovery). Both can achieve the same global $r$ through appropriate variance scaling.
\end{proof}

\paragraph{Implication.} Global correlation alone cannot predict decision utility. Audits must separately measure within-prompt metrics.

\subsection{Optimal Routing and Its Limits}

Consider the routing problem: given budget $b$, choose query probability $\pi(W)$ to maximize:
\[
\max_{\pi: \E[\pi(W)] \leq b} \E\big[Y_1 + \pi(W) \cdot (Y_2 - Y_1)\big],
\]
where $Y_1$ is oracle value of judge-selected response and $Y_2$ is oracle-best value.

\begin{proposition}[Optimal Routing]
\label{prop:routing}
Define $\Delta(W) = \E[Y_2 - Y_1 \mid W]$. The optimal router is:
\[
\pi^*(W) = \mathbf{1}\{\Delta(W) \geq \tau\},
\]
where threshold $\tau$ is chosen to satisfy the budget constraint $\E[\pi^*(W)] = b$.
\end{proposition}

\begin{proof}
The objective is linear in $\pi(W)$:
\[
\E[Y_1] + \E[\pi(W) \cdot \Delta(W)].
\]
Maximizing $\E[\pi(W) \cdot \Delta(W)]$ subject to $\E[\pi(W)] \leq b$ is a knapsack problem. The greedy solution queries prompts in decreasing order of $\Delta(W)$ until budget is exhausted, which is equivalent to thresholding.
\end{proof}

\begin{corollary}[Routing Bound]
\label{cor:bound}
Let $\sigma^2_\Delta = \mathrm{Var}(\Delta(W))$. The maximum improvement of any router over random allocation is:
\[
V(\pi^*) - V(\pi_{\text{rand}}) \leq \sqrt{\sigma^2_\Delta} \cdot \sqrt{b(1-b)}.
\]
If $\sigma^2_\Delta \approx 0$, all routers achieve approximately the same value as random allocation.
\end{corollary}

\paragraph{Application to our setting.} We observe $\mathrm{Var}(\Delta(W)) / \mathrm{Var}(Y_2 - Y_1) \approx 0.03$ when $W$ includes margin and level. This means observables explain only 3\% of gain variance, severely limiting any router's improvement over random.

\subsection{The Value-of-Information Calibration Gap}

Define two calibration properties:

\begin{definition}[Confidence Calibration]
A judge is \emph{confidence-calibrated} if $\Pr(\text{correct} \mid \text{margin} = m)$ is monotone increasing in $m$.
\end{definition}

\begin{definition}[VOI Calibration]
A judge is \emph{VOI-calibrated} (value-of-information calibrated) if $\E[\text{gain} \mid \text{margin} = m]$ is monotone decreasing in $m$.
\end{definition}

\begin{proposition}[Calibration Gap]
\label{prop:calib}
Confidence calibration does not imply VOI calibration. Specifically, if margin correlates with prompt difficulty (harder prompts $\to$ higher margin), then a confidence-calibrated judge can fail VOI calibration.
\end{proposition}

\begin{proof}
By the decomposition:
\[
\E[\text{gain} \mid m] = \Pr(\text{wrong} \mid m) \times \E[\text{gap} \mid \text{wrong}, m].
\]

Confidence calibration ensures $\Pr(\text{wrong} \mid m)$ decreases with $m$. But if $\E[\text{gap} \mid \text{wrong}, m]$ increases with $m$ (because high-margin prompts are harder, leaving more room for improvement), the product can be non-monotone or even increasing.

In our data: $\mathrm{Corr}(\text{margin}, \text{correct}) = +0.063$ (confidence-calibrated), but $\mathrm{Corr}(\text{margin}, \text{gap} \mid \text{wrong}) = +0.03$ (harder prompts), yielding $\mathrm{Corr}(\text{margin}, \text{gain}) = -0.033 \approx 0$.
\end{proof}

\paragraph{Implication.} A judge can ``know when it is uncertain'' (confidence calibration) without ``knowing when querying the oracle helps'' (VOI calibration). These are distinct properties, and margin captures the former but not the latter.

\section{Formal Statement: Level vs Directional Validity}
\label{app:level-directional}

We formalize the distinction between level validity (addressed by CJE) and directional validity (this paper).

\paragraph{Decomposition.} Let $S_{x,i}$ denote the judge score and $O_{x,i}$ the oracle label for prompt $x$ and candidate $i$:
\begin{align}
    S_{x,i} &= \mu^S_x + \varepsilon^S_{x,i} \quad \text{(Judge: prompt mean + candidate deviation)} \\
    O_{x,i} &= \mu^O_x + \varepsilon^O_{x,i} \quad \text{(Oracle: prompt mean + candidate deviation)}
\end{align}

\paragraph{Global correlation} conflates two sources:
\begin{itemize}
    \item $\mathrm{Corr}(\mu^S, \mu^O)$ (agreement on prompt difficulty; between-prompt)
    \item $\mathrm{Corr}(\varepsilon^S, \varepsilon^O)$ (agreement on within-prompt quality; within-prompt)
\end{itemize}

\begin{definition}[Level Validity]
A judge satisfies level validity if $\E_\pi[f(S)] = \E_\pi[Y]$ for some calibration function $f$. That is, calibrated judge scores yield unbiased policy value estimates.
\end{definition}

\begin{definition}[Directional Validity]
A judge satisfies directional validity if $\mathrm{Corr}(\varepsilon^S_{x,i}, \varepsilon^O_{x,i}) > \tau$ for some threshold $\tau$. That is, within-prompt rankings are correct with sufficient reliability.
\end{definition}

\begin{theorem}[Independence of Validity Levels]
\label{thm:independence-level-directional}
Level validity and directional validity are independent properties:
\begin{enumerate}
    \item A judge can satisfy level validity without directional validity: averaging over prompts cancels within-prompt errors, yielding unbiased policy estimates despite poor instance-level decisions.
    \item A judge can satisfy directional validity without level validity: correct within-prompt rankings do not imply calibrated absolute scores.
\end{enumerate}
\end{theorem}

\paragraph{Why Averaging Helps Policy Estimation But Not Instance Selection.}

For policy-level estimation:
\[
\mathrm{Var}(\hat{V}(\pi)) = \frac{1}{n^2} \sum_x \mathrm{Var}(S_{x,\pi(x)}) = \frac{1}{n} \times \text{Average}(\mathrm{Var}(S_{x,i})) \to 0 \text{ as } n \to \infty
\]
This is the standard $\sqrt{n}$ convergence. Even with large individual variance, the policy mean converges.

For instance-level selection at a single prompt $x$:
\[
P(\text{correct} \mid x) = f(\rwithin, \text{tie rate})
\]
There is no averaging: each prompt is a single decision. The full within-prompt noise directly determines decision quality.

\paragraph{Empirical Variance Decomposition (This Paper).}

\begin{table}[h]
\centering
\begin{tabular}{lrr}
\toprule
\textbf{Component} & \textbf{Judge} & \textbf{Oracle} \\
\midrule
Between-prompt variance & 74\% & 81\% \\
Within-prompt variance & 26\% & 19\% \\
\bottomrule
\end{tabular}
\end{table}

The high between-prompt variance creates global correlation ($r = 0.47$) without implying within-prompt agreement ($\rwithin = 0.27$). This is the structure that allows level validity to coexist with directional failure.

\section{Monotone Calibration Preserves Rankings}
\label{app:monotone}

A natural question is whether score calibration could repair directional validity failures. The following lemma formalizes why monotone calibration methods, including the isotonic calibration used by CJE, cannot improve within-prompt rankings.

\begin{lemma}[Monotone Calibration Preserves Argmax]
\label{lem:monotone}
Let $f: \mathbb{R} \to \mathbb{R}$ be strictly monotone increasing. For any set of scores $\{S_i\}_{i=1}^n$:
\[
\arg\max_i f(S_i) = \arg\max_i S_i
\]
except for tie-breaking when $S_i = S_j$ for some $i \neq j$.
\end{lemma}

\begin{proof}
Suppose $i^* = \arg\max_i S_i$, so $S_{i^*} \geq S_j$ for all $j$. By strict monotonicity of $f$:
\begin{itemize}
    \item If $S_{i^*} > S_j$, then $f(S_{i^*}) > f(S_j)$
    \item If $S_{i^*} = S_j$, then $f(S_{i^*}) = f(S_j)$ (tie preserved)
\end{itemize}
Thus $f(S_{i^*}) \geq f(S_j)$ for all $j$, with equality only when $S_{i^*} = S_j$.
\end{proof}

\paragraph{Implications.} This lemma has three key consequences:

\begin{enumerate}
    \item \textbf{CJE compatibility:} CJE's \emph{monotone-only} isotonic calibration is monotone in the scalar score by construction. Lemma~\ref{lem:monotone} explains why this mode achieves level validity without affecting within-prompt rankings: calibration corrects scale but preserves ordering. CJE's two-stage calibration can incorporate response-varying covariates (e.g., length) and may re-rank candidates through the learned index; this can affect directional behavior, but is not guaranteed to improve it.

    \item \textbf{Directional validity is orthogonal:} No monotone transformation of judge scores can improve within-prompt accuracy. Directional validity must come from the judge's \emph{raw ranking ability}, not post-hoc calibration.

    \item \textbf{Ties as the exception:} The only case where calibration could affect rankings is tie-breaking. With 67\% tie rates, this exception covers the majority of decisions, but random tie-breaking cannot systematically improve accuracy.
\end{enumerate}

This formalizes a key insight: \textbf{level and directional validity are addressed by fundamentally different mechanisms}. Level validity can be repaired via calibration; directional validity cannot.

\section{Theoretical Baseline: Gaussian Recovery Formula}
\label{app:gaussian}

Under a clean ``noisy-but-linear'' within-prompt model, there is a closed-form relationship between within-prompt correlation and expected recovery.

\paragraph{Setup.} Within a prompt with $n$ candidates:
\begin{itemize}
    \item Oracle utilities $O_i \sim \mathcal{N}(0, 1)$ i.i.d.
    \item Judge scores $S_i = \rho \cdot O_i + \sqrt{1-\rho^2} \cdot Z_i$ where $Z_i \sim \mathcal{N}(0, 1)$ i.i.d.
    \item $\mathrm{Corr}(S_i, O_i) = \rho$ (within-prompt correlation)
\end{itemize}

\paragraph{Result.} Let the judge select $i^* = \arg\max_i S_i$. Then:
\[
\E[O_{i^*}] = \rho \cdot \E[\max_i O_i]
\]

If the random baseline has $\E[O] = 0$, then:
\[
\text{Recovery} = \frac{\E[O_{i^*}]}{\E[\max_i O_i]} = \rho
\]

\textbf{In the idealized Gaussian/no-tie world, expected best-of-$n$ recovery equals within-prompt correlation.}

\paragraph{Why This Matters.} This formula provides a clean theoretical baseline:
\begin{itemize}
    \item If the world were ``nice'' (continuous scores, Gaussian utilities, no ties), $\rwithin$ would directly predict recovery
    \item Deviations from this baseline are the \emph{explanandum}: ties, discretization, heteroskedasticity, nonlinearity, prompt-dependent noise
\end{itemize}

\paragraph{Our Empirical Results as Deviations.}

\begin{table}[h]
\centering
\begin{tabular}{lrrr}
\toprule
\textbf{Quantity} & \textbf{Theoretical} & \textbf{Empirical} & \textbf{Gap} \\
\midrule
$\rwithin$ & --- & 0.27 & --- \\
Recovery (if Gaussian) & 27\% & 21.0\% & $-$6.0\% \\
Tie rate & 0\% & 67\% & +67\% \\
\bottomrule
\end{tabular}
\end{table}

The 6.0 percentage point gap between theoretical (27\%) and empirical (21.0\%) recovery is explained by:
\begin{enumerate}
    \item \textbf{Discretization}: 67\% tie rate means most comparisons are uninformative
    \item \textbf{Non-Gaussian tails}: Extreme values may have different noise structure
    \item \textbf{Heteroskedasticity}: Judge noise may vary with oracle utility
\end{enumerate}

This baseline clarifies that our findings are not merely ``the judge is noisy''; they identify \emph{specific mechanisms} (discretization, ties) that degrade decision utility beyond what correlation alone would predict.

\section{Evaluation Prompts}
\label{app:prompts}

\subsection{Pointwise Judge Prompt}

The following system prompt is used for pointwise scoring (0--100 scale):

\begin{verbatim}
You are an expert AI evaluator. Rate the quality of the following
response to a user query on a scale from 0 to 100.

Consider:
- Helpfulness: Does the response address the user's request?
- Accuracy: Is the information correct and well-reasoned?
- Clarity: Is the response well-organized and easy to understand?
- Completeness: Does the response fully address the query?

Provide only a single integer score from 0 to 100, with no explanation.
\end{verbatim}

\subsection{Pairwise Judge Prompt}

The actual pairwise protocol used for Table~\ref{tab:pairwise} and the confidence analyses allowed \texttt{A}, \texttt{B}, or \texttt{TIE} and requested a confidence rating. The implementation enforces the same schema via structured parsing.

\begin{verbatim}
You are an AI evaluator. Compare two responses to the same user
question and decide which is better. You must always make a choice -
ties are not allowed unless the responses are truly indistinguishable.
\end{verbatim}

\begin{verbatim}
<task>
Compare the following two AI assistant responses to the same user
question. Decide which response is better based on relevance,
helpfulness, clarity, and completeness.
</task>

<question>
{prompt}
</question>

<response_a>
{response_a}
</response_a>

<response_b>
{response_b}
</response_b>

<instruction>
Which response is better? Respond with one of:
- "A" if Response A is better
- "B" if Response B is better
- "TIE" only if the responses are truly indistinguishable in quality

Also rate your confidence (1-5, where 5 is very confident).

Format your response EXACTLY as:
CHOICE: [A/B/TIE]
CONFIDENCE: [1-5]
</instruction>
\end{verbatim}

\subsection{Probability Elicitation Prompt}

For elicited probability estimates (used in confidence analysis):

\begin{verbatim}
You are an expert AI evaluator. Compare the following two responses
(A and B) to a user query.

What is the probability that Response A is better than Response B?

Provide only a single number between 0.0 and 1.0, with no explanation.
\end{verbatim}

\section{Multi-Judge Detailed Results}
\label{app:multi-judge}

Table~\ref{tab:multi-judge-full} shows complete metrics for all five judges evaluated in Section~\ref{sec:multi-judge}.

\begin{table}[h]
\centering
\caption{Full multi-judge results on the shared 500-prompt sample (486--500 complete prompts per judge after missing-score filtering).}
\label{tab:multi-judge-full}
\begin{tabular}{lrrrrr}
\toprule
\textbf{Judge} & \textbf{Global $r$} & \textbf{Within $r$} & \textbf{Tie Rate} & \textbf{Sign Agr} & \textbf{Recovery} \\
\midrule
GPT-5.2 & 0.872 & 0.701 & 24.8\% & 76.5\% & 69.4\% \\
Claude Sonnet 4 & 0.587 & 0.417 & 59.7\% & 77.7\% & 47.7\% \\
GPT-4.1-mini & 0.563 & 0.465 & 57.8\% & 74.8\% & 43.6\% \\
Gemini-2.5-flash & 0.465 & 0.269 & 55.2\% & 66.4\% & 23.8\% \\
Llama-3.3-70b & 0.308 & 0.232 & 60.7\% & 60.4\% & 18.6\% \\
\bottomrule
\end{tabular}
\end{table}

\paragraph{Calibration ablation.} We tested CJE's FlexibleCalibrator with three modes: monotone (isotonic regression), two-stage (spline + isotonic), and two-stage with response length as covariate. Pure monotone calibration improves level validity (global $r$) but not directional validity (recovery). Adding covariates can break ties and improve recovery, but at the cost of sign agreement; the tie-breaking is not quality-aligned:

\begin{table}[h]
\centering
\caption{Calibration ablation: average change from raw scores across all judges.}
\label{tab:calibration-ablation}
\begin{tabular}{lrrrr}
\toprule
\textbf{Mode} & \textbf{$\Delta$ Global $r$} & \textbf{$\Delta$ Within $r$} & \textbf{$\Delta$ Sign Agr} & \textbf{$\Delta$ Recovery} \\
\midrule
Monotone & +23\% & +10\% & +0.8\% & $-$0.3\% \\
Two-stage & +24\% & +10\% & +0.3\% & $-$0.3\% \\
Two-stage + length & +38\% & +17\% & $-$3.7\% & +2.6\% \\
\bottomrule
\end{tabular}
\end{table}

This does not contradict Lemma~\ref{lem:monotone}. The lemma is about rank preservation: a monotone transform cannot change within-prompt argmax decisions except through pre-existing ties. Pearson correlations such as global $r$ and within-prompt $r$ are not rank-invariant, so they can change under monotone rescaling even when the underlying rankings, and therefore PCS/Recovery absent tie effects, do not.

\section{Efficient Influence Functions for Decision Value Estimation}
\label{app:eif}

We develop the semiparametric efficiency theory for estimating decision value under partial oracle labeling \cite{tsiatis2006semiparametric}. This provides: (i) doubly robust estimators with valid confidence intervals, and (ii) a principled derivation of optimal oracle allocation. We test whether margin-based routing approximates the optimal design; empirically, it does not (\Cref{sec:efficient}).

\subsection{Setup and Notation}

For each prompt $x$, we observe $K$ candidate responses with judge scores $S_{1:K}$ and (potentially missing) oracle utilities $O_{1:K}$. Let $W = (x, S_{1:K}, \text{features})$ denote the fully observed covariates.

Oracle labels are expensive, so we only observe them when queried:
\begin{itemize}[leftmargin=*]
    \item $R \in \{0,1\}$: indicator for whether oracle was queried for this prompt
    \item $\pi(W) = \Pr(R=1 \mid W)$: query probability (known by design or modeled)
\end{itemize}

\paragraph{Observed data.} $(W, R, R \cdot O_{1:K})$ for $n$ prompts.

\paragraph{Assumption (MAR by design).} $R \perp O_{1:K} \mid W$. This holds when oracle allocation depends only on observed features (e.g., judge scores, margin).

\subsection{Decision Value Functional}

A \emph{selector} $\delta: W \mapsto \{1, \ldots, K\}$ maps observed features to a chosen candidate. Examples:
\begin{align}
    \delta_{\text{judge}}(W) &= \argmax_i S_i \quad \text{(judge-greedy)} \\
    \delta_{\text{rand}}(W) &\sim \text{Uniform}\{1,\ldots,K\} \quad \text{(random)} \\
    \delta_{\text{oracle}}(W) &= \argmax_i O_i \quad \text{(oracle-optimal, counterfactual)}
\end{align}

The \textbf{decision value} of selector $\delta$ is:
\begin{equation}
V(\delta) = \E[O_{\delta(W)}].
\end{equation}

This is the expected oracle utility when using $\delta$ for selection, exactly what best-of-$n$ optimizes.

\subsection{Efficient Influence Function for \texorpdfstring{$V(\delta)$}{V(delta)}}

Let $m_\delta(W) = \E[O_{\delta(W)} \mid W]$ be the conditional expectation of oracle utility given covariates.

\begin{proposition}[EIF for Decision Value]
Under the nonparametric MAR model, the efficient influence function for $V(\delta)$ is:
\begin{equation}
\phi_\delta(Z) = \frac{R}{\pi(W)}\Big(O_{\delta(W)} - m_\delta(W)\Big) + m_\delta(W) - V(\delta).
\end{equation}
\end{proposition}

This is the standard missing-data EIF instantiated on the decision functional.

\paragraph{Doubly robust estimator.} The AIPW estimator \cite{bang2005doubly,tsiatis2006semiparametric} based on this EIF is:
\begin{equation}
\label{eq:aipw}
\widehat{V}(\delta) = \frac{1}{n} \sum_{t=1}^{n} \left[ \widehat{m}_\delta(W_t) + \frac{R_t}{\widehat{\pi}(W_t)} \Big( O_{\delta(W_t)} - \widehat{m}_\delta(W_t) \Big) \right].
\end{equation}

\paragraph{Properties.}
\begin{itemize}[leftmargin=*]
    \item \textbf{Double robustness}: Consistent if either $\widehat{\pi}$ or $\widehat{m}_\delta$ is correctly specified.
    \item \textbf{Semiparametric efficiency}: Achieves the efficiency bound when both nuisances are estimated at $o(n^{-1/4})$ rates.
    \item \textbf{ML compatibility}: Cross-fitting yields $\sqrt{n}$-consistent inference with flexible ML estimators \cite{chernozhukov2018double}.
\end{itemize}

\subsection{EIF for Recovery Rate}

Recovery is defined as:
\begin{equation}
\text{Rec} = \frac{V(\delta_{\text{judge}}) - V(\delta_{\text{rand}})}{V(\delta_{\text{oracle}}) - V(\delta_{\text{rand}})} = \frac{\psi_1 - \psi_0}{\psi_2 - \psi_0} =: g(\psi_0, \psi_1, \psi_2).
\end{equation}

Each $\psi_k = V(\delta_k)$ has EIF $\phi_k$ of the form above. By the delta method, the EIF for recovery is:
\begin{equation}
\phi_{\text{Rec}} = \frac{\partial g}{\partial \psi_0} \phi_0 + \frac{\partial g}{\partial \psi_1} \phi_1 + \frac{\partial g}{\partial \psi_2} \phi_2,
\end{equation}
where the partial derivatives are:
\begin{align}
\frac{\partial g}{\partial \psi_1} &= \frac{1}{\psi_2 - \psi_0}, \\
\frac{\partial g}{\partial \psi_2} &= -\frac{\psi_1 - \psi_0}{(\psi_2 - \psi_0)^2}, \\
\frac{\partial g}{\partial \psi_0} &= \frac{\psi_1 - \psi_2}{(\psi_2 - \psi_0)^2}.
\end{align}

This yields efficient estimation and asymptotically valid confidence intervals for recovery under oracle subsampling.

\subsection{Optimal Proxy Oracle Allocation}

Consider estimating $\psi = \E[f(W, O)]$ with the DR estimator. The asymptotic variance depends on the oracle query probability $\pi(\cdot)$.

\begin{proposition}[Variance-Minimizing Design]
Subject to budget constraint $\E[\pi(W)] \leq b$, the variance-minimizing query probability is:
\begin{equation}
\pi^*(W) \propto \sqrt{\Var(f(W, O) \mid W)},
\end{equation}
truncated to $[0, 1]$ with a Lagrange multiplier to satisfy the budget.
\end{proposition}

This is the Neyman allocation principle applied to decision evaluation \cite{neyman1934representative}.

\paragraph{Connection to margin-based routing.} For $f(W, O) = O_{\delta_{\text{judge}}(W)}$, we might hypothesize that the conditional variance $\Var(O_{\delta(W)} \mid W)$ is larger when:
\begin{itemize}[leftmargin=*]
    \item The judge margin is small (uncertain which candidate is best)
    \item Multiple candidates have similar judge scores (tie or near-tie)
\end{itemize}

If this held, margin-based routing would approximate the EIF-optimal design. However, this connection is \textbf{empirical, not guaranteed}: it requires that judge uncertainty actually predicts oracle outcome variance.

\paragraph{Empirical validation required.} Whether margin-based allocation achieves variance reduction depends on the data-generating process. In \Cref{sec:efficient}, we test this hypothesis and find that in our dataset, $\text{Corr}(\text{margin}, \text{oracle variance}) = 0.162$ (positive, not negative). Margin-based allocation does not reduce variance; oracle-optimal allocation (which uses the true conditional variance) achieves 45--77\% reduction. This highlights the gap between theoretical optimality and practical proxies.

\subsection{Non-Regularity of the Optimal Selector}

We emphasize that the above theory applies to evaluation of \emph{fixed} selectors $\delta$, not inference on the \emph{optimal} selector $\delta^* = \argmax_\delta V(\delta)$.

Inference on $\sup_{\delta \in \mathcal{D}} V(\delta)$ or the argmax policy is \textbf{non-regular} in fully nonparametric models unless:
\begin{itemize}[leftmargin=*]
    \item The policy class $\mathcal{D}$ is restricted (finite or parametric)
    \item A margin condition holds (no near-ties in oracle value)
    \item The objective is smoothed (softmax instead of argmax)
\end{itemize}

For this reason, we focus on evaluating the \emph{given} judge-based selector rather than claiming to find or do inference on an optimal routing rule.

\end{document}